\documentclass{article}

\usepackage[preprint]{neurips_2023}
\usepackage{csquotes}

\usepackage{amsmath,amsfonts,bm}

\def\eqref#1{equation~\ref{#1}}

\def\1{\bm{1}}

\DeclareMathAlphabet{\mathsfit}{\encodingdefault}{\sfdefault}{m}{sl}
\SetMathAlphabet{\mathsfit}{bold}{\encodingdefault}{\sfdefault}{bx}{n}

\def\cC{{\mathcal{C}}}

\def\cO{{\mathcal{O}}}

\newcommand{\R}{\mathbb{R}}

\DeclareMathOperator*{\argmin}{arg\,min}

\usepackage{url}

\usepackage[utf8]{inputenc} 
\usepackage[T1]{fontenc}    
\usepackage{url}            
\usepackage{booktabs}       

\usepackage{amsmath, amsfonts, amssymb}
\usepackage{amsthm}   
   
\usepackage{nicefrac}       
\usepackage{microtype}      
\usepackage{color,xcolor}         
\usepackage{xcolor}
\usepackage{makecell}
\usepackage{soul} 
\definecolor{myyellow}{rgb}{1,1,0.7}
\sethlcolor{myyellow}

\newcommand{\myref}[1]{$\left(\ref{#1}\right)$} 

\newcommand{\eqdef}{\coloneqq}

\usepackage{mathtools}   
\usepackage{algorithm}
\usepackage{algorithmic}

\usepackage{multirow}

\usepackage{colortbl}
\definecolor{bgcolor}{rgb}{0.93,0.99,1}
\definecolor{bgcolor2}{rgb}{0.8,1,0.8}
\definecolor{bgcolor3}{rgb}{0.50,0.90,0.50}

\usepackage{relsize} 
\newcommand{\algname}[1]{{\sf\red\relscale{0.90}#1}\xspace}

\newcommand{\gammaM}{{\gamma}}
\newcommand{\tauM}{{\tau}}

\newcommand{\nclients}{{M}}
\newcommand{\iclient}{m}
\newcommand{\kstep}{t}
\newcommand{\kcohort}{{C}}

\newcommand{\Koper}{{{H}}}
\newcommand{\localsteps}{K}
\newcommand{\comm}{T}

\newcommand{\localsolmk}{y_{\iclient}^{\star,\kstep}}
\newcommand{\lastlocittermk}{y_{\iclient}^{\localsteps,\kstep}}
\newcommand{\lastlocitterk}{y^{\localsteps,\kstep}}
\newcommand{\localsolk}{y^{\star,\kstep}}

\newcommand{\set}{S^{t}}

\newcommand{\localfun}{\psi^t}
\newcommand{\localfuni}{\localfun_\iclient}

\newtheorem{theorem}{Theorem}[section]
\newtheorem{lemma}[theorem]{Lemma}
\newtheorem{corollary}[theorem]{Corollary}

\newtheorem{assumption}{Assumption}

\usepackage[flushleft]{threeparttable} 

\usepackage{tcolorbox}
\usepackage{pifont}
\definecolor{mydarkgreen}{RGB}{39,130,67}
\definecolor{mydarkred}{RGB}{192,25,25}

\newcommand{\red}{\color{mydarkred}}
\usepackage{color}

\usepackage[colorinlistoftodos,bordercolor=orange,backgroundcolor=orange!20,linecolor=orange,textsize=scriptsize]{todonotes}
\newcommand{\michal}[1]{\todo[inline]{\textbf{Michal: }#1}}

\usepackage{thmtools} 
\usepackage{thm-restate}

\definecolor{lgray}{rgb}{0.95,0.95,0.95}
\definecolor{yel}{rgb}{1,0.98,0.92}
\definecolor{mydarkblue}{rgb}{0,0.2,0.9}

\definecolor{mydarkred}{rgb}{0.8,0.0,0.0}

\definecolor{mydarkgreen}{rgb}{0,0.55,0}

\usepackage{xspace}
\usepackage{scalefnt}
\definecolor{myorange}{RGB}{255,100,0}
\newcommand{\algn}[1]{{\sf\color{mydarkred}\scalefont{0.96}{#1}}\xspace}

\newcommand{\sqnorm}[1]{\left\| #1 \right\|^2}
\newcommand{\Exp}[1]{\mathbb{E}\!\left[ #1 \right]}

\newcommand{\Rop}{\mathcal{P}}
\newcommand{\Ropp}{\widetilde{\mathcal{R}}}
\newcommand{\oma}{\omega_{\mathrm{ran}}}

\usepackage{hyperref}
\usepackage{graphicx}
\usepackage{booktabs}
\newtheorem*{theorem*}{Theorem}
\newtheorem*{lemma*}{Lemma}
\newtheorem*{corollary*}{Corollary}
\newcommand{\Kx}{\localsteps}
\newcommand{\Cx}{\kcohort}
\newcommand{\Mx}{\nclients}
\newcommand{\Tx}{\comm}
\newcommand{\qm}{q_\iclient}
\newcommand{\prm}{p_\iclient}
\newcommand{\conic}{\omega}

\usepackage[utf8]{inputenc} 
\usepackage[T1]{fontenc}    
\usepackage{hyperref}       
\usepackage{url}            
\usepackage{booktabs}       
\usepackage{amsfonts}       
\usepackage{nicefrac}       
\usepackage{microtype}      
\usepackage{xcolor}         

\title{Improving Accelerated Federated Learning with Compression and Importance Sampling}

\author{%
  Micha\l{} Grudzie\'n\\
    Departament of Mathematics\\
  University of Oxford\\
\And
Grigory Malinovsky\\
AI Initiative\\
KAUST\\
\And 
Peter Richt\'arik\\
AI Initiative\\
KAUST
}

\begin{document}

\maketitle

\begin{abstract}
	Federated Learning is a collaborative training framework that leverages heterogeneous data distributed across a vast number of clients. Since it is practically infeasible to request and process all clients during the aggregation step, partial participation must be supported. In this setting, the communication between the server and clients poses a major bottleneck. To reduce communication loads, there are two main approaches: compression and local steps. Recent work by~\citet{ProxSkip} introduced the new \algname{ProxSkip} method, which achieves an accelerated rate using the local steps technique. Follow-up works successfully combined local steps acceleration with partial participation \citep{grudzien2023can, condat2023tamuna} and gradient compression \citep{condat2022provably}. In this paper, we finally present a complete method for Federated Learning that incorporates all necessary ingredients: Local Training, Compression, and Partial Participation. We obtain state-of-the-art convergence guarantees in the considered setting. Moreover, we analyze the general sampling framework for partial participation and derive an importance sampling scheme, which leads to even better performance. We experimentally demonstrate the advantages of the proposed method in practice.
\end{abstract}

\section{Introduction}
Federated Learning (FL)~\citep{FEDLEARN, FLblog2017} is a distributed machine learning paradigm that allows multiple devices or clients to collaboratively train a shared model without transferring their raw data to a central server. In traditional machine learning, data is typically gathered and stored in a central location for training a model. However, in Federated Learning, each client trains a local model using its own data and shares only the updated model parameters with a central server or aggregator. The server then aggregates the updates from all clients to create a new global model, which is then sent back to each client to repeat the process~\citep{FedAvg2016}.

This approach has gained significant attention due to its ability to address the challenges of training machine learning models on decentralized and sensitive data~\cite{FL2017-AISTATS}. Federated Learning enables clients to preserve their privacy and security by keeping their data local and not sharing it with the central server. This approach also reduces the need for large-scale data transfers, thereby minimizing communication costs and latency~\citep{li2020federated}.

Federated Learning poses several challenges such as data heterogeneity, communication constraints, and ensuring the privacy and security of the data~\citep{kairouz2021advances}. Researchers in this field have developed novel optimization algorithms to address these challenges and to enable efficient aggregation of the model updates from multiple clients~\citep{FieldGuide2021}. Federated Learning has been successfully applied to various applications, including healthcare~\citep{SplitLearning2}, finance~\citep{long2020federated}, and Internet of Things (IoT) devices~\citep{khan2021federated}. 

This work considers the standard formulation of Federated Learning as a finite sum minimization problem:
\begin{align}
	\label{main_form}
	\min_{x\in \R^d}\left[ f(x)\eqdef \frac{1}{M} \sum_{m=1}^{M} f_m(x) \right]
\end{align}
where $M$ is the number of clients/devices. Each function $f_m(x) = \mathbb{E}_{\xi\sim \mathcal{D}_m}\left[l(x,\xi)\right]$ represents the average loss, calculated via the loss function $l$, of the model parameterized by $x\in \R^d$ over the training data $\mathcal{D}_m$ stored by client $m \in \left[M\right] \eqdef \left\lbrace 1,\ldots, M \right\rbrace.$ 

\subsection{Federated Averaging}
The method known as Federated Averaging (\algname{FedAvg}), proposed by \citet{FL2017-AISTATS}, is a widely used technique that specifically addresses the challenges of practical federated environments while solving problem~\ref{main_form}. \algname{FedAvg} is based on Gradient Descent (\algname{GD}), but applies four modifications: Client Sampling (CS), Data Sampling (DS), Local Training (LT), Communication Compression (CC). 

The \algname{FedAvg} training process occurs over several communication rounds. At the start of each round~$t$, a subset or cohort $S^t \subset [M]$ of clients with a size of $C^t = |S^t|$ is selected to participate in that round's training. The aggregating server then sends the current model version, $x^t$, to all clients $m \in S^t$. Each client $m \in S^t$ performs $K$ iterations of \algname{SGD} on its local loss function, $f_m$, using minibatches $\mathcal{B}_m^{k, t} \subseteq \mathcal{D}_m \text { of size } b_m=\left|\mathcal{B}_m^{k, t}\right| \text { for } k=0, \ldots, K-1$ initialized with $x^t$. Afterward, all participating clients compress their updated models and send these compressed updates to the server to aggregate into a new model version, $x^{t+1}$. The entire process is then repeated. This generalized scheme is described in~\citet{grudzien2023can}.

Each of the four \algname{FedAvg} modifications, Client Sampling (CS), Data Sampling (DS), Local Training (LT), and Communication Compression can be independently turned on or off, or used in various combinations. For instance, if $C^t = M$ for all rounds, then all clients take part in every round, resulting in the deactivation of CS. Similarly, when $b_m = |D_m|$ for each client $m \in [M]$, every client employs all their data to compute the local gradient estimator needed for \algname{SGD}, which leads to the deactivation of DS. Additionally, when $K$ is set to $1$, each participating client performs only one \algname{SGD} step, causing LT to be turned off.  Lastly, If the compression operator is set to the identity, then each client transmits complete updates, resulting in compression being turned off. If all four modifications are disabled, \algname{FedAvg} becomes equivalent to vanilla gradient descent (\algname{GD}).

\subsection{Data Sampling}
The seminal works discussed earlier illustrate the practical benefits of the novel approach, i.e., \algname{FedAvg} method, but they lack theoretical analysis and associated guarantees. Given that \algname{FedAvg} comprises four distinct components, it is expedient to analyze these techniques in isolation to achieve a deeper comprehension of each of them.

Due to the close association of unbiased data sampling techniques with the stochastic approximation literature dating back to the works of~\citet{RobbinsMonro:1951, nemirovsky1983problem, nemi09, bottou2018optimization}, it is not unexpected that CS is comparatively well comprehended. For instance, \citet{gower2019sgd} have scrutinized variations of SGD that back almost any unbiased CS mechanism in the smooth strongly convex area, while \citet{ES-SGD-nonconvex} have analyzed those in the smooth nonconvex region. Furthermore, \citet{PAGE-AS} have proposed and studied oracle-optimal versions of \algname{SGD} that support almost any unbiased CS and DS mechanisms in the smooth nonconvex region, drawing upon the previous works of \citet{PAGE2021}, \citet{SPIDER}, and \citet{SARAH, SARAH-nonconvex}. The CS with variance reduction techniques is widely analyzed by~\citet{sigma_k}.

\subsection{Client Sampling}

As distributed learning gained popularity, researchers began to examine Client Sampling strategies for improving communication efficiency \citep{wu2022node} and ensuring robustness and security during aggregation \citep{so2021securing}. Empirical studies of Client Sampling strategies can be found in the literature, such as \citet{fraboni2021clustered, charles2021large, huang2022stochastic}. Optimal Client Sampling strategies under various conditions have been theoretically analyzed in works such as \citet{wang2022client} and \citet{OptClientSampling2020}. The cyclic patterns of client participation are studied in~\citet{malinovsky2023federated, cho2023convergence}. While Client Sampling shares similarities with data sampling, it also has distinct characteristics that need to be taken into account. 

\subsection{Communication Compression}
Communication Compression is a valuable component in distributed optimization, as it allows each client to transmit a compressed or quantized version of its update, $\Delta^t_{m}$, instead of the entire update vector. This can lead to significant bandwidth savings by reducing the number of bits transmitted over the network. Various operators have been proposed for compressing update vectors, such as stochastic quantization \citep{alistarh2017qsgd}, random sparsification \citep{wangni2017gradient, StichNIPS2018-memory}, and alternative methods~\citep{DoubleSqueeze2019}. 

The utilization of unbiased compressors can decrease the amount of bits that clients transmit per round. However, it can also cause an increase in the variance of the stochastic gradients, which leads to a slower overall convergence~\citep{DCGD, stich2020communication}. To address this issue, \citet{DIANA} proposed \algname{DIANA}, an algorithm that uses control iterates to diminish the variance resulting from gradient compression with unbiased compression operators. This approach guarantees fast convergence. \algname{DIANA} has been examined and extended in various scenarios \citep{DIANA2, DIANA+,DIANA++,D-DIANA,ADIANA} and is a valuable tool for utilizing gradient compression.

The article presents the application of compression techniques in Federated Learning, as discussed in  \citet{basu2019qsparse, reisizadeh2020fedpaq, haddadpour2021federated}. The mechanism of compressing iterates is studied in \citet{GDCI, DFPMCI2019}. Additionally, \citet{malinovsky2022federated} and \citet{ sadiev2022federated} investigate the application of compression with random reshuffling in Federated Learning.

\subsection{Five Generations of Local Training}
Local Training (LT) is a crucial aspect of Federated Learning (FL) models, where each participating client performs multiple local optimization steps before synchronization of parameters. In the smooth strongly convex regime, we will provide a concise overview of the theoretical advancements made in understanding LT. \citet{ProxSkip-VR} categorized LT methods into five generations - heuristic, homogeneous, sublinear, linear, and accelerated - each progressively enhancing the previous one in significant ways.

\textbf{1st (heuristic) generation of LT methods.} 
Although the ideas behind LT were previously utilized in various machine learning fields~\citet{Povey2015,SparkNet2016}, it gained significant attention as a communication acceleration technique following seminal paper introducing the \algname{FedAvg} algorithm \citep{FL2017-AISTATS}. However, their work, along with previous research, lacked any theoretical justification. As a result, LT-based heuristics dominated the field's initial development until the \algname{FedAvg} paper and lacked any theoretical guarantees.

\textbf{2nd (homogeneous) generation of LT methods )}The second generation of LT methods offers guarantees, but their analysis relies on various data homogeneity assumptions. These assumptions include bounded gradients, which require $\left\|\nabla f_m(x)\right\| \leq c$ for all $m \in[M]$ and $x \in \mathbb{R}^d$ \citep{Li-local-bounded-grad-norms--ICLR2020}, or  bounded gradient dissimilarity, i.e., requiring $\frac{1}{M} \sum_{m=1}^M\left\|\nabla f_m(x)\right\|^2 \leq c\|\nabla f(x)\|^2 \text { for all } x \in \mathbb{R}^d$ \citep{LocalDescent2019}. The reasoning behind such assumptions is that in the extreme case when all local functions are identical, running \algname{GD} independently and in parallel on all clients without any communication or averaging would make \algname{GD} communication-efficient. Based on this, as we increase heterogeneity, taking multiple local steps should still be beneficial as long as the number of steps is limited. However, using bounded dissimilarity assumptions is highly problematic as they are not met even in some of the simplest function classes, such as strongly convex quadratics \citep{localGD, localSGD}. Furthermore, due to the highly heterogeneous and non-i.i.d nature of real-world federated learning datasets, relying on strong assumptions like data/gradient homogeneity for analyzing LT methods is both mathematically dubious and practically insignificant. Several authors have analyzed various LT methods under such assumptions and obtained rates \citep{Yu-local-homogeneous-2019,Li2019-local-homogeneous, FedAvg-nonIID}

\textbf{3rd (sublinear) generation of LT methods}. The third generation LT theory successfully eliminated the need for data homogeneity assumptions, as demonstrated by \citet{localGD, localSGD}. However, subsequent studies by \citet{woodworth2020minibatch} and \citet{glasgow2022sharp} showed that \algname{LocalGD} with DS (\algname{LocalSGD}) has communication complexity that is no better than minibatch \algname{SGD} in the heterogeneous data setting. Furthermore, \citet{LFPM} analyzed LT methods for general fixed point problems and \citet{koloskova2020unified} studied decentralized accepts of Local Training. While removing the need for data homogeneity assumptions was a significant advancement, the results were rather pessimistic, indicating that LT-enhanced \algname{GD}, or \algname{LocalGD}, has a sublinear convergence rate, which is inferior to vanilla \algname{GD}'s linear convergence rate \citep{Blake2020}. The effect of server-side stepsizes is analyzed in \citet{malinovsky2022server, charles2020outsized} 

\textbf{4th (linear) generation of LT methods. } The focus of the fourth generation of LT methods was to develop linearly converging versions of LT algorithms by addressing the problem of client drift, which was identified as the reason behind the previous generation's subpar performance compared to \algname{GD}. The first method to successfully mitigate client drift and achieve a linear convergence rate was \algname{Scaffold}, as proposed by \citet{SCAFFOLD}. Other approaches to achieve the same effect were later introduced by \citet{LSGDunified2020} and \citet{FEDLIN}. While obtaining a linear rate under standard assumptions was a significant achievement, these methods still have a slightly higher communication complexity than vanilla \algname{GD} and at best equal to that of \algname{GD}.

\textbf{5th (accelerated) generation of LT methods.}
\citet{ProxSkip} have recently introduced the \algname{ProxSkip} method, which represents a new and simple approach to Local Training that results in provable communication acceleration in the smooth strongly convex regime, even when dealing with heterogeneous data. Specifically, in cases where each $f_m$ is $L$-smooth and $\mu$-strongly convex, \algname{ProxSkip} can solve \ref{main_form} in 
$\mathcal{O}(\sqrt{L / \mu} \log 1 / \varepsilon)$ communication rounds, a significant improvement over the $\mathcal{O}(L / \mu \log 1 / \varepsilon)$ complexity of GD. This accelerated communication complexity has been shown to be optimal by \citet{scaman2019optimal}. \citet{ProxSkip} have also introduced several extensions to \algname{ProxSkip}, including a flexible data sampling framework and decentralized version. As a result of these developments, other new methods that can achieve communication acceleration using Local Training are proposed. 

The initial sequel article by \citet{ProxSkip-VR} presents a broad variance reduction structure for the \algname{ProxSkip} approach. In addition, \citet{RandProx} applies the \algname{ProxSkip} methodology to complex splitting schemes that involve the sum of three operators in a forward-backward setting. Besides,\citet{sadiev2022communication} and \citet{maranjyan2022gradskip} improve the computational complexity of the \algname{ProxSkip} method while maintaining its accelerated communication acceleration.\citet{condat2023tamuna} introduces accelerated Local Training methods that allow Client Sampling based on the \algname{ProxSkip} method, and \citet{grudzien2023can} provide an accelerated method with Client Sampling based on \algname{RandProx} method with primal and dual updates. However, these methods are limited as they only work with a uniform distribution of clients. \algname{CompressedScaffnew} \citep{condat2022provably} is the first LT method to achieve accelerated communication complexity while utilizing compression of updates. However, it works only with permutation-based compressors \citep{PermK}, and it is not compatible with a broad range of unbiased compressors. Permutation-based compressors demand synchronization of compression patterns during the aggregation stage, which is not feasible for Federated Learning settings due to privacy aspects.

\section{Contributions}
	Our work is based on the observation that none of the 5th generation Local Training (LT) methods currently support both Client Sampling (CS) and Communication Compression (CC). This raises the question of whether it is possible to design a method that can benefit from communication acceleration via LT while also supporting CS and utilizing Communication Compression techniques. 
	
	At this point, we are prepared to summarize the crucial observations and contributions made in our work.
	\begin{itemize}
			\item To the best of our knowledge, we provide the first LT  method that successfully combines communication acceleration through local steps, Client Sampling techniques, and Communication Compression for a wide range of unbiased compressors. Our proposed algorithm for distributed optimization and federated learning is the first of its kind to utilize both strategies in combination, resulting in a doubly accelerated rate. Our method based on method \algname{5GCS} \citep{grudzien2023can} benefits from the two acceleration mechanisms provided by Local Training and compression in the Client Sampling regime, exhibiting improved dependency on the condition number of the functions and the dimension of the model, respectively.
		
	\item 	In this paper, we investigate a comprehensive Client Sampling framework based on the work of ~\citet{tyurin2022sharper}, which we then apply to the \algname{5GCS} method proposed by \citet{grudzien2023can}. This approach enables us to analyze a wide range of Client Sampling techniques, including both sampling with and without replacement and it recovers previous results for uniform distribution. The framework also allows us to determine optimal probabilities, which results in improved communication.

	\end{itemize}

\section{Preliminaries}

\subsection{Method's description}
This section provides a description of the proposed methods in this paper. Specifically, we consider two algorithms (Algorithm 1 and Algorithm 2), both of which share the same core idea. At the beginning of the training process, we initialize several parameters, including the starting point $x^0$, the dual (control) iterates $u^0_1,\ldots,u^0_M$, the primal (server-side) stepsize, and $M$ dual (local) stepsizes. Additionally, we choose a sampling scheme $\mathbf{S}$ for Algorithm 1 or a type of compressor $\mathcal{Q}$ for Algorithm 2. Once all parameters are set, we commence the iteration cycle.

At the start of each communication round, we sample a cohort (subset) of clients according to a particular scheme. The server then computes the intermediate model $\hat{x}^t$ and sends this point to each client in the cohort. Once each client receives the model $\hat{x}^t$, the worker uses it as a starting point for solving the local sub-problem defined in Equation~\ref{asn:argmin}. After approximately solving the local sub-problem, each client computes the gradient of the local function at the approximate solution $\nabla F_m(y^{K,t}_m)$ and, based on this information, each client forms and sends an update to the server, either with or without compression. The server then aggregates the received information from workers and updates the global model $x^{t+1}$ and additional variables if necessary. This process repeats until convergence.

\subsection{Assumptions}
We begin by adopting the standard assumption in convex optimization \citep{NesterovBook}.
\begin{assumption}
	\label{assm:L-smooth-conv} The functions $f_m$ are $L_m$-smooth and $\mu_m$-strongly convex for all $m \in\left\lbrace1,...,M\right\rbrace.$
\end{assumption}	
All of our theoretical results will rely on this standard assumption in convex optimization. To recap, a continuously differentiable function $\phi:\mathbb{R}^d\rightarrow\mathbb{R}$ is $L$-smooth if $\phi(x) - \phi(y) - \langle\nabla\phi(y), x-y\rangle \leq \frac{L}{2}\|x-y\|^2$ for all $x, y \in \mathbb{R}^d$, and $\mu$-strongly convex if $\phi(x) - \phi(y) - \langle\nabla\phi(y), x-y\rangle \geq \frac{\mu}{2}\|x-y\|^2$ for all $x, y \in \mathbb{R}^d$, $\overline{L} = \frac{1}{M}\sum_{m=1}^{M} L_m$ and $L_{\max} = \max_m L_m$.

Our method employs the same reformulation of problem~$\ref{main_form}$ as it is used in \citet{grudzien2023can}, which we will now describe. Let $H:\mathbb{R}^d\rightarrow\mathbb{R}^{Md}$ be the linear operator that maps $x\in\mathbb{R}^d$ to the vector $(x, \ldots, x)\in\mathbb{R}^{Md}$ consisting of $M$ copies of $x$. First, note that $F_m(x):=\frac{1}{M}\left(f_m(x) - \frac{\mu_m}{2}\|x\|^2\right)$ is convex and $L_{F,m}$-smooth, where $L_{F,m}:=\frac{1}{M}(L_m-\mu_m)$. Furthermore, we define $F:\mathbb{R}^{Md}\rightarrow\mathbb{R}$ as $F\left(x_1, \ldots, x_M\right):=\sum_{m=1}^M F_m\left(x_m\right)$.

Having introduced the necessary notation, we state the following formulation in the lifted space, which is equivalent to the initial problem~$\ref{main_form}$:
\begin{align}
	\label{lifted_form}
	x^{\star}=\underset{x \in \mathbb{R}^d}{\arg \min }\left[f(x):=F(H x)+\frac{\mu}{2}\|x\|^2\right],
\end{align}
where $\mu = \frac{1}{M}\sum_{m=1}^{M}\mu_m$.

The dual problem to~\ref{lifted_form} has the following form:
\begin{align}
	\label{dual_form}
	u^{\star}=\underset{u \in \mathbb{R}^{M d}}{\arg \max }\left(\frac{1}{2 \mu}\left\|\sum_{m=1}^M u_m\right\|^2+\sum_{m=1}^M F_m^*\left(u_m\right)\right),
\end{align}
where $F_m^*$ is the Fenchel conjugate of $F_m$, defined by $F_m^*(y):=\sup _{x \in \mathbb{R}^d}\left\{\langle x, y\rangle-F_m(x)\right\}$. Under Assumption~\ref{assm:L-smooth-conv} , the primal and dual problems have unique optimal solutions $x^{\star}$ and $u^{\star}$, respectively. 

Next, we consider the tool of analyzing sampling schemes, which is Weighted AB Inequality from \citet{tyurin2022sharper}. Let $\Delta^M:=\left\{\left(p_1, \ldots, p_M\right) \in \mathbb{R}^M \mid p_1, \ldots, p_M \geq 0, \sum_{m=1}^M p_m=1\right\}$ be the standard simplex and $(\Omega, \mathcal{F}, \mathbf{P})$ a probability space.
\begin{assumption}\label{weighted}
	
	(Weighted AB Inequality). Consider the random mapping $\mathbf{S}: \left\lbrace 1,\ldots, M \right\rbrace \times \Omega \rightarrow \left\lbrace 1,\ldots, M \right\rbrace    $, which we call “sampling”. For each sampling we consider the random mapping that we call estimator  $S : \R^d \times \ldots \times \R^d \times \Omega \to \R^d$, such that $\Exp{S(a_1, \ldots , a_M; \psi)} = \frac{1}{M}\sum_{m=1}^M a_m$ for all $a_1,\ldots,a_M \in \R^d$. Assume that there exist $A,B \geq 0$ and weights $(w_1,\ldots,w_M) \in \Delta^M$ such that
	$$ \Exp{\sqnorm{S(a_1,\ldots,a_M;\psi) -  \frac{1}{M}\sum_{m=1}^M a_m}} \leq \frac{A}{M^2}\sum_{m=1}^M \frac{\sqnorm{a_m}}{w_m} - B\sqnorm{\frac{1}{M}\sum_{m=1}^M a_m} , \forall a_m \in\R^d.$$

\end{assumption}

Furthermore, it is necessary to specify the number of local steps to solve sub-problem \ref{asn:argmin}. To maintain the generality and arbitrariness of local solvers, we use an inequality that ensures the accuracy of the approximate solutions of local sub-problems is sufficient. It should be noted that the assumption below covers a broad range of optimization methods, including all linearly convergent algorithms.

\begin{assumption}
	\label{assm:localtr}
	(Local Training). Let $\left\{\mathcal{A}_1, \ldots, \mathcal{A}_M\right\}$ be any Local Training (LT) subroutines for minimizing functions $\left\{\psi_1^t, \ldots, \psi_M^t\right\}$ defined in~\ref{asn:argmin}, capable of finding points $\left\{y_1^{K, t}, \ldots, y_M^{K, t}\right\}$ in $K$ steps, from the starting point $y_m^{0, t}=\hat{x}^t$ for all $m \in\{1, \ldots, M\}$, which satisfy the inequality
	\begin{eqnarray*}
	 	\sum_{\iclient=1}^\Mx\frac{4}{\tauM_\iclient^2}\frac{\mu_\iclient L^2_{F_\iclient}}{3M}\sqnorm{\lastlocittermk-\localsolmk}+\sum_{\iclient=1}^\Mx\frac{L_{F_\iclient}}{\tauM_\iclient^2}\sqnorm{\nabla\localfuni (\lastlocittermk)}\leq\sum_{\iclient=1}^\Mx\frac{\mu_\iclient}{6M}\sqnorm{\hat{x}^t-\localsolmk},
	\end{eqnarray*}
	where $y_m^{\star, t}$ is the unique minimizer of $\psi_m^t$, and $\tauM_\iclient\geq\frac{8\mu_\iclient}{3\Mx}$.
\end{assumption}
Finally, we need to specify the class of compression operators. We consider the class of unbiased compressors with conic variance \citep{condat2021murana}. 
\begin{assumption}
	\label{asn:conic}
	(Unbiased compressor). A randomized mapping $\mathcal{Q}: \mathbb{R}^d \rightarrow \mathbb{R}^d$ is an unbiased compression operator $\left(\mathcal{Q} \in \mathbb{U}(\omega)\right.$ for brevity) if for some $\omega \geq 0$ and $\forall x \in \mathbb{R}^d$
	$$\mathbb{E} \mathcal{Q}(x)=x, \text{	(Unbiasedness) }\quad \mathbb{E}\|\mathcal{Q}(x)-x\|^2 \leq \omega\|x\|^2
	\text{ (Conic variance) }.$$
\end{assumption}

\section{Communication Compression}
In this section we provide convergence guarantees for the Algorithm~\ref{alg:5GC-CC} (\algname{5GCS-CC}), which is the version that combines Local Training, Client Sampling and Communication Compression.
\begin{theorem} \label{thm:inexactCompr}
Let Assumption~\ref{assm:L-smooth-conv} hold. Consider Algorithm~\ref{alg:5GC-CC} (\algname{5GCS-CC}) with the LT solvers $\mathcal{A}_m$ satisfying Assumption~\ref{assm:localtr} and compression operators $\mathcal{Q}_m$ satisfying Assumption~\ref{asn:conic}. Let $\tau = \tau_m$ for all $m\in \left\lbrace 1,\ldots, M \right\rbrace$ and  $\frac{1}{\tauM}-\gammaM(\Mx+\conic\frac{\Mx}{\Cx})\geq \frac{4}{\tauM^2}\frac{\mu}{3\Mx}$, for example: $\tauM\geq\frac{8\mu}{3\Mx}$ and $\gammaM=\frac{1}{2\tauM\left(\Mx+\conic\frac{\Mx}{\Cx}\right)}$. Then for the Lyapunov function
	\begin{equation*}
		  	\Psi^{\kstep}\eqdef \frac{1}{\gammaM}\sqnorm{x^{\kstep}-x^\star}+\frac{\Mx}{\Cx}\left(\omega+1\right)\left(\frac{1}{\tauM}+\frac{1}{L_{F,\max}}\right)\sum_{m=1}^{M}\sqnorm{u^{\kstep}_m-u_m^\star},
	\end{equation*}
	the iterates  satisfy
	$		\Exp{\Psi^{\Tx}}\leq (1-\rho)^\Tx \Psi^0,
	$	where 
	$
	\rho\eqdef  \min\left\{\frac{\gammaM\mu}{1+\gammaM\mu},\frac{C}{M(1+\omega)}\frac{\tauM}{(L_{F,{\max}}+\tauM)}\right\}<1.
	$  
	
\end{theorem}
Next, we derive the communication complexity for Algorithm~\ref{alg:5GC-CC} (\algname{5GCS-CC}).
\begin{corollary}\label{cor:5GCSC}
	Choose any $0<\varepsilon<1$ and $\tauM = \frac{8}{3}\sqrt{\mu L_{\max}\left(\frac{\conic+1}{\Cx}\right)\frac{1}{\Mx\left(1+\frac{\conic}{\Cx}\right)}}$ and $\gammaM=\frac{1}{2\tauM\Mx\left(1+\frac{\conic}{\Cx}\right)}$. In order to guarantee $\Exp{\Psi^{\Tx}}\leq \varepsilon \Psi^0$, it suffices to take 
	\begin{eqnarray}
		  	T \geq	   
		\cO\left(\left(\frac{\Mx}{\Cx}\left(\conic+1\right)+\left(\sqrt{\frac{\conic}{\Cx}}+1\right)\sqrt{\left(\conic+1\right)\frac{\Mx}{\Cx}\frac{L}{\mu}}\right) \log \frac{1}{\varepsilon}\notag \right)
	\end{eqnarray}
	communication rounds.
\end{corollary}
Note, if no compression is used ($\omega = 0$) we recover the rate of \algname{5GCS}: $\mathcal{O}\left(\left(\nicefrac{M}{C}+\sqrt{\nicefrac{ML}{ C \mu}}\right)\log \frac{1}{\varepsilon}\right)$. 

\section{General Client Sampling}
\begin{algorithm*}[!t]
	\caption{\algn{5GCS-CC}}
	\footnotesize
	\begin{algorithmic}[1]\label{alg:5GC-CC}
		\STATE  \textbf{Input:} initial primal iterates $x^0\in\mathbb{R}^d$; initial dual iterates $u_1^0, \dots,u_{\Mx}^0 \in\mathbb{R}^d$; primal stepsize $\gammaM>0$; dual stepsize $\tauM>0$; cohort size $\Cx\in \{1,\dots,\Mx\}$
		\STATE  \textbf{Initialization:}  $v^0\eqdef \sum_{\iclient=1}^\Mx u_\iclient^0$  \hfill {\color{gray} \footnotesize $\diamond$ The server initiates $v^{0}$ as the sum of the initial dual iterates} 
		\FOR{communication round $t=0, 1, \ldots$} 
		\STATE Choose a cohort $\set\subset \{1,\ldots,\Mx\}$ of clients of cardinality $\Cx$, uniformly at random \hfill {\color{gray} \footnotesize $\diamond$  CS step} 
		\STATE Compute $\hat{x}^{t} = \frac{1}{1+\gammaM\mu} \left(x^t - \gammaM v^t\right)$ and broadcast it to the clients  in the cohort 
		\FOR{$\iclient\in \set$}
		\STATE Find $\lastlocittermk$ as the final point after $\Kx$ iterations of some local optimization algorithm $\mathcal{A}_\iclient$, initiated with $y_\iclient^0=\hat{x}^t$, for solving the optimization problem \hfill {\color{gray} \footnotesize $\diamond$ Client $\iclient$ performs $\Kx$ LT steps} 
		\begin{eqnarray}
			\label{asn:argmin}
			  	\lastlocittermk \approx \argmin \limits_{y\in\mathbb{R}^d}\left\{\localfuni(y) \eqdef  F_\iclient(y)+\frac{\tauM_m}{2} \sqnorm{y-\left(\hat{x}^\kstep+\frac{1}{\tauM_m}u_\iclient^t\right)}\right\}\label{localprob}
		\end{eqnarray}
		\STATE Compute $\bar{u}_\iclient^{t+1}=\nabla F_\iclient(\lastlocittermk)$ 
		\STATE  $u_\iclient^{t+1}= u_\iclient^t + \tfrac{1}{1+\conic}\tfrac{\kcohort}{\Mx} Q_\iclient\left(\bar{u}_\iclient^{t+1}-u_\iclient^t\right)$ 
		\STATE Send $Q_\iclient\left(\bar{u}^{t+1}_m-u_m^t\right)$ to the server. \hfill {\color{gray} \footnotesize $\diamond$ Server updates $u_\iclient^{t+1}$}
		\ENDFOR
		\STATE 
		\FOR{$\iclient\in\{1, \ldots,\Mx\}\backslash \set$}
		\STATE $u_{\iclient}^{t+1}\eqdef u_{\iclient}^t $ \hfill {\color{gray} \footnotesize $\diamond$ Non-participating clients do nothing} 
		\ENDFOR
		\STATE $v^{t+1} \eqdef v^t + \frac{1}{1+\omega} \frac{C}{M}\sum_{m=1}^{M} \mathcal{Q}_m\left(\bar{u}^{t+1}_m-u_m^t\right) $ \hfill {\color{gray} \footnotesize $\diamond$ The server keeps  $v^{t+1}$ as the sum of the dual iterates} 
		\STATE  $x^{t+1} \eqdef \hat{x}^{t}- \gammaM \frac{\Mx}{\Cx}\left(1+\conic\right) (v^{t+1}-v^t)$ \hfill {\color{gray} \footnotesize $\diamond$ The server updates the primal iterate} 
		\ENDFOR
	\end{algorithmic}
\end{algorithm*}

In this section we analyze Algorithm~\ref{alg:5GCS-AB} (\algname{5GCS-AB}). First, we introduce a general result for all sampling schemes that can satisfy Assumption~\ref{weighted}

\begin{theorem}
\label{thm:5GCS-AB}
Let Assumption~\ref{assm:L-smooth-conv} hold. Consider Algorithm~\ref{alg:5GCS-AB} with sampling scheme $\mathbf{S}$ satisfying Assumption~\ref{weighted} and LT solvers $\mathcal{A}_m$ satisfying Assumption~\ref{assm:localtr}. Let the inequality hold $\frac{1}{\tauM_\iclient}-\left(\gammaM\left(1-B\right)\Mx +\gammaM\frac{A}{w_\iclient}\right)\geq \frac{4}{\tauM_\iclient^2}\frac{\mu_\iclient}{3M}$, e.g. $\tauM_\iclient\geq\frac{8\mu_\iclient}{3\Mx}$ and $\gammaM\leq\frac{1}{2\tauM_\iclient\left(\left(1-B\right)\Mx +\frac{A}{w_\iclient}\right)}$.  Then for the Lyapunov function 
	\begin{equation*}
		  	\Psi^{\kstep}\eqdef \frac{1}{\gammaM}\sqnorm{x^{t}-x^\star}+\sum_{\iclient=1}^\Mx\left(1+\qm\right)\left(\frac{1}{\tauM_\iclient}+\frac{1}{L_{F_\iclient}}\right)\sqnorm{u_\iclient^{t}-u_\iclient^\star},
	\end{equation*}
	the iterates of  the method satisfy
	\begin{equation*}
	 	\Exp{\Psi^{t+1}}\leq  \max\left\{\frac{1}{1+\gammaM\mu} ,\max_{m}\left[\frac{ L_{F_\iclient}+\frac{\qm}{1+\qm}\tauM_\iclient}{L_{F_\iclient}+\tauM_\iclient}\right] \right\}\Exp{\Psi^{t}},
	\end{equation*}
	where $\qm=\frac{1}{\widehat{p}_\iclient}-1$ and $\widehat{p}_\iclient$ is probability that $m$-th client is participating. 
\end{theorem}
The obtained result is contingent upon the constants $A$ and $B$, as well as the weights $w_m$ specified in Assumption~\ref{weighted}. Furthermore, the rate of the algorithm is influenced by $\widehat{p}_m$, which represents the probability of the $m$-th client participating. This probability is dependent on the chosen sampling scheme $\mathbf{S}$ and needs to be derived separately for each specific case. In main part of the work we consider two important examples: Multisampling and Independent Sampling.

\subsection{Sampling with Replacement (Multisampling)}

Let $\underline{p} = \left(p_1, p_2, \ldots, p_M\right)$ be probabilities summing up to 1 and let $\chi_m$ be the random variable equal to $m$ with probability $p_m$. Fix a cohort size $C \in\{1,2, \ldots, M \}$ and let $\chi_1, \chi_2, \ldots, \chi_C$ be independent copies of $\chi$. Define the gradient estimator via
\begin{equation}
	\label{eq:swr}
	S\left(a_1, \ldots, a_n, \psi, \underline{p} \right):=\frac{1}{C} \sum_{m=1}^C \frac{a_{\chi_m}}{M p_{\chi_m}}
\end{equation}
By utilizing this sampling scheme and its corresponding estimator, we gain the flexibility to assign arbitrary probabilities for client participation while also fixing the cohort size. However, it is important to note that under this sampling scheme, certain clients may appear multiple times within the cohort.
\begin{lemma}
	\label{lem:multi}
	The Multisampling with estimator~\ref{eq:swr} satisfies the Assumption~\ref{weighted} with $A= B = \frac{1}{C}$ and $w_m=\prm$.
\end{lemma}

\begin{algorithm*}[!t]
	\caption{\algn{5GCS-AB}}
	\footnotesize
	\begin{algorithmic}[1]\label{alg:5GCS-AB}
		\STATE  \textbf{Input:} initial primal iterate $x^0\in\mathbb{R}^d$; initial dual iterates $u_1^0, \dots,u_{\Mx}^0 \in\mathbb{R}^d$; primal stepsize $\gammaM>0$; dual stepsizes $\tauM_\iclient>0$; $\omega\in \mathcal{D}_\Mx$
		\STATE  \textbf{Initialization:}  $v^0\eqdef \sum_{\iclient=1}^\Mx u_\iclient^0$  \hfill {\color{gray} \footnotesize $\diamond$ The server initiates $v^{0}$ as the sum of the initial dual iterates} 
		\FOR{communication round $t=0, 1, \ldots$} 
		\STATE Sample a cohort $\set\subset \{1,\ldots,\Mx\}$ of clients according to sampling scheme $\mathbf{S}$
		\STATE Compute $\hat{x}^{t} = \frac{1}{1+\gammaM\mu} \left(x^t - \gammaM v^t\right)$ and broadcast it to the clients  in the cohort 
		\FOR{$\iclient\in \set$}
		\STATE Find $\lastlocittermk$ as the final point after $\Kx$ iterations of some local optimization algorithm $\mathcal{A}_\iclient$, initiated with $y_\iclient^0=\hat{x}^t$, for solving the optimization problem \hfill {\color{gray} \footnotesize $\diamond$ Client $\iclient$ performs $\Kx$ LT steps} 
		\begin{eqnarray}
			\label{asn:argmin1}
			  	\lastlocittermk \approx \argmin \limits_{y\in\mathbb{R}^d}\left\{\localfuni(y) \eqdef  F_\iclient(y)+\frac{\tauM_\iclient}{2} \sqnorm{y-\left(\hat{x}^\kstep+\frac{1}{\tauM_\iclient}u_\iclient^t\right)}\right\}\label{localprob}
		\end{eqnarray}
		\STATE Compute $\bar{u}_\iclient^{t+1}=\nabla F_\iclient(\lastlocittermk)$ 
		\STATE Update $u_{\iclient}^{t+1}=\bar{u}_{\iclient}^{t+1}.$ 
		\ENDFOR
		\FOR{$\iclient\in  \{1,\ldots,\Mx\}\setminus\set$}
		\STATE Update $u_{\iclient}^{t+1}=u_{\iclient}^{t}.$ 
		\ENDFOR
		\STATE  $x^{t+1} \eqdef \hat{x}^{t}- \gammaM \Mx \cdot S({u}_1^{t+1}-u^t_1,\ldots,{u}_\Mx^{t+1}-u^t_\Mx;\omega)$ \hfill {\color{gray} \footnotesize $\diamond$ The server updates the primal iterate} 
		\STATE 
		\STATE $v^{t+1}=\sum_{\iclient=1}^{\Mx}u_\iclient^{t+1}$
		\ENDFOR
	\end{algorithmic}
\end{algorithm*}
Now we are ready to formulate the theorem.

\begin{theorem} \label{thm:5GCMult}
	Let Assumption~\ref{assm:L-smooth-conv} hold. Consider Algorithm~\ref{alg:5GCS-AB} (\algname{5GCS-AB}) with Multisampling and estimator~\ref{eq:swr} satisfying Assumption~\ref{weighted} and LT solvers $\mathcal{A}_m$ satisfying Assumption~\ref{assm:localtr}. Let the inequality hold $\frac{1}{\tauM_\iclient}-\left(\gammaM\left(1-\frac{1}{\Cx}\right)\Mx +\gammaM\frac{1}{\Cx \prm}\right)\geq \frac{4}{\tauM_\iclient^2}\frac{\mu_\iclient}{3\Mx}$ , e.g. $\tauM_\iclient\geq\frac{8\mu_\iclient}{3\Mx}$ and $\gammaM\leq\frac{1}{2\tauM_\iclient\left(\left(1-\frac{1}{\Cx}\right)\Mx +\frac{1}{\Cx \prm}\right)}$.  Then for the Lyapunov function 
	\begin{equation*}
		  	\Psi^{\kstep}\eqdef \frac{1}{\gammaM}\sqnorm{x^{t}-x^\star}+\sum_{\iclient=1}^\Mx\frac{1}{\widehat{p}_\iclient}\left(\frac{1}{\tauM_\iclient}+\frac{1}{L_{F_\iclient}}\right)\sqnorm{u_\iclient^{t}-u_\iclient^\star},
	\end{equation*}
	the iterates of  the method satisfy
	\begin{equation*}
	 	\Exp{\Psi^{t+1}}\leq  \max\left\{\frac{1}{1+\gammaM\mu} , \max_{m}\left[\frac{ L_{F_\iclient}+\left(1-\widehat{p}_\iclient\right)\tauM_\iclient}{L_{F_\iclient}+\tauM_\iclient} \right\}\right]\Exp{\Psi^{t}},
	\end{equation*}
	where $\widehat{p}_\iclient = 1-\left(1-\prm\right)^\Cx$ is probability that $m$-th client is participating. 
	
\end{theorem}
Regrettably, it does not appear to be feasible to obtain a closed-form solution for the optimal probabilities and stepsizes when $C>1$. Nevertheless, we were able to identify a specific set of parameters for a special case where $C = 1$. Furthermore, even in this particular case, the solution is not exact. However, based on the Brouwer fixed-point theorem \citep{cite-key}, a solution for $p_m$ and $\tau_m$ in Corollary \ref{cor:5GCS1} exists.

\begin{corollary}\label{cor:5GCS1}
	Suppose $\Cx=1$. Choose any $0<\varepsilon<1$ and $p_\iclient=\frac{\sqrt{L_{F,\iclient}+\tauM_\iclient}}{\sum_{\iclient=1}^\Mx \sqrt{L_{F,\iclient}+\tauM_\iclient}}$, and $\tauM_\iclient=\frac{8}{3}\sqrt{\overline{L}\mu\Mx}p_\iclient$.  In order to guarantee $\Exp{\Psi^{\Tx}}\leq \varepsilon \Psi^0$, it suffices to take 
	\begin{eqnarray*}
		  	T \geq	 
		\max\left\{1+\frac{16}{3}\sqrt{\frac{\overline{L}\Mx}{\mu}},\frac{3}{8}\sqrt{\frac{\overline{L}\Mx}{\mu}}+\Mx\right\}\log \frac{1}{\varepsilon}
	\end{eqnarray*}
	communication rounds.
\end{corollary}
To address the challenge posed by the inexact solution, we have also included the exact formulas for the parameters. While this set of parameters may not offer the optimal complexity, it can still be valuable in certain cases.
\begin{corollary}\label{cor:5GCS2}
	Suppose $C=1$. Choose any $0<\varepsilon<1$ and $p_\iclient=\frac{\sqrt{\frac{L_\iclient}{\Mx}}}{\sum_{\iclient=1}^\Mx \sqrt{\frac{L_\iclient}{\Mx}}}$, and $\tauM_\iclient=\frac{8}{3}\sqrt{\overline{L}\mu\Mx}p_\iclient$.  In order to guarantee $\Exp{\Psi^{\Tx}}\leq \varepsilon \Psi^0$, it suffices to take 
	\begin{eqnarray*}
		  	T \geq	
		\max\left\{1+\frac{16}{3}\sqrt{\frac{\overline{L}\Mx}{\mu}},\frac{3}{8}\sqrt{\frac{\overline{L}\Mx}{\mu}}+\frac{\sum_{\iclient=1}^\Mx \sqrt{L_\iclient}}{\sqrt{L_{\min}}}\right\}\log \frac{1}{\varepsilon}
	\end{eqnarray*}
	communication rounds. Note that $L_{\min} = \min_m L_m$.
\end{corollary}

\subsection{Sampling without Replacement (Independent Sampling)}
In the previous example, the server had the ability to control the cohort size and assign probabilities for client participation. However, in practical settings, the server lacks control over these probabilities due to various technical conditions such as internet connections, battery charge, workload, and others. Additionally, each client operates independently of the others. Considering these factors, we adopt the Independent Sampling approach. Let us formally define such a scheme. To do so, we introduce the concept of independent and identically distributed (i.i.d.) random variables:
\begin{eqnarray*}
	\chi_\iclient=\begin{cases}
		1 & \text{with probability }p_m\\
		0 & \text{with probability }1-p_m,
	\end{cases}
\end{eqnarray*}
for all $m\in[M]$, also take $S^t\eqdef\left\{m\in[M]|\chi_\iclient=1\right\}$ and $\underline{p} = \left(p_1,\ldots, p_M\right)$ .  The corresponding estimator for this sampling has the following form:  
\begin{equation}
	\label{eq:1}
	S(a_1,\ldots,a_M,\psi, \underline{p})\eqdef \frac{1}{\Mx}\sum_{\iclient\in S}\frac{a_m}{p_{m}},
\end{equation}
The described sampling scheme with its estimator is called the Independence Sampling. Specifically, it is essential to consider the probability that all clients communicate, denoted as $\Pi_{m=1}^M\prm$, as well as the probability that no client participates, denoted as $\Pi_{m=1}^M(1-\prm)$. It is important to note that $\sum_{m=1}^{M}\prm$ is not necessarily equal to $1$ in general. Furthermore, the cohort size is not fixed but rather random, with the expected cohort size denoted as $\Exp{\set}=\sum_{m=1}^{M}\prm$.
\begin{lemma}
	\label{lem:ind}
	The Independent Sampling with estimator~\ref{eq:1} satisfies the Assumption~\ref{weighted} with $A=\frac{1}{\sum_{m}^{M}\frac{\prm}{1-\prm}}$, $B=0$   and $w_m=\frac{\frac{\prm}{1-\prm}}{\sum_{m=1}^{M}\frac{\prm}{1-\prm}}$.
\end{lemma}
Now we are ready to formulate the convergence guarantees and derive communication complexity.  
\begin{theorem}\label{thm:5GCSINDS}
	Consider Algorithm~\ref{alg:5GCS-AB} with Independent Sampling with estimator~\ref{eq:1} satisfying Assumption~\ref{weighted} and LT solver satisfying Assumption~\ref{assm:localtr}. Let the inequality hold $\frac{1}{\tauM_\iclient}-\left(\gammaM \Mx +\gammaM\frac{1-\prm}{\prm}\right)\geq \frac{4}{\tauM_\iclient^2}\frac{\mu_\iclient}{3M}$, e.g. $\tauM_\iclient\geq\frac{8\mu_\iclient}{3\Mx}$ and $\gammaM\leq\frac{1}{2\tauM_\iclient\left(\Mx +\frac{1-\prm}{\prm}\right)}$.  Then for the Lyapunov function 
	\begin{equation*}
		  	\Psi^{\kstep}\eqdef \frac{1}{\gammaM}\sqnorm{x^{t+1}-x^\star}+\sum_{\iclient=1}^\Mx\frac{1}{\prm}\left(\frac{1}{\tauM_\iclient}+\frac{1}{L_{F_\iclient}}\right)\sqnorm{u_\iclient^{t+1}-u_\iclient^\star},
	\end{equation*}
	the iterates of  the method satisfy
	\begin{equation*}
	 	\Exp{\Psi^{t+1}}\leq  \max\left\{\frac{1}{1+\gammaM\mu} ,\max_{m}\left[\frac{ L_{F_\iclient}+\left(1-\prm\right)\tauM_\iclient}{L_{F_\iclient}+\tauM_\iclient}\right] \right\}\Exp{\Psi^{t}},
	\end{equation*}
	where $\prm$ is probability that $m$-th client is participating. 
\end{theorem}
\begin{corollary}\label{cor:5GCSINDS}
	Choose any $0<\varepsilon<1$ and $p_\iclient$ can be estimated but not set, then set $\tauM_\iclient =\frac{8}{3}\sqrt{\frac{\bar{L}\mu}{M\sum_{\iclient=1}^\Mx \prm}}$ and $\gammaM=\frac{1}{2\tauM_\iclient\left(\Mx +\frac{1-\prm}{\prm}\right)}$.  In order to guarantee $\Exp{\Psi^{\Tx}}\leq \varepsilon \Psi^0$, it suffices to take 
	\begin{eqnarray*}
		  	T \geq	 
		\max\left\{1+\frac{16}{3}\sqrt{\frac{\overline{L}\Mx}{\mu\sum_{\iclient=1}^\Mx \prm}}\left(1+\frac{1-\prm}{\Mx\prm}\right),\max_{m}\left[\frac{3L_{F_m}}{8\prm}\sqrt{\frac{\Mx\sum_{\iclient=1}^\Mx \prm}{\overline{L}\mu}}+\frac{1}{\prm}\right]\right\}\log \frac{1}{\varepsilon}
	\end{eqnarray*}
	communication rounds.
\end{corollary}
\begin{figure*}[t]
	\centering
	\begin{tabular}{cc}
		\includegraphics[scale=0.415]{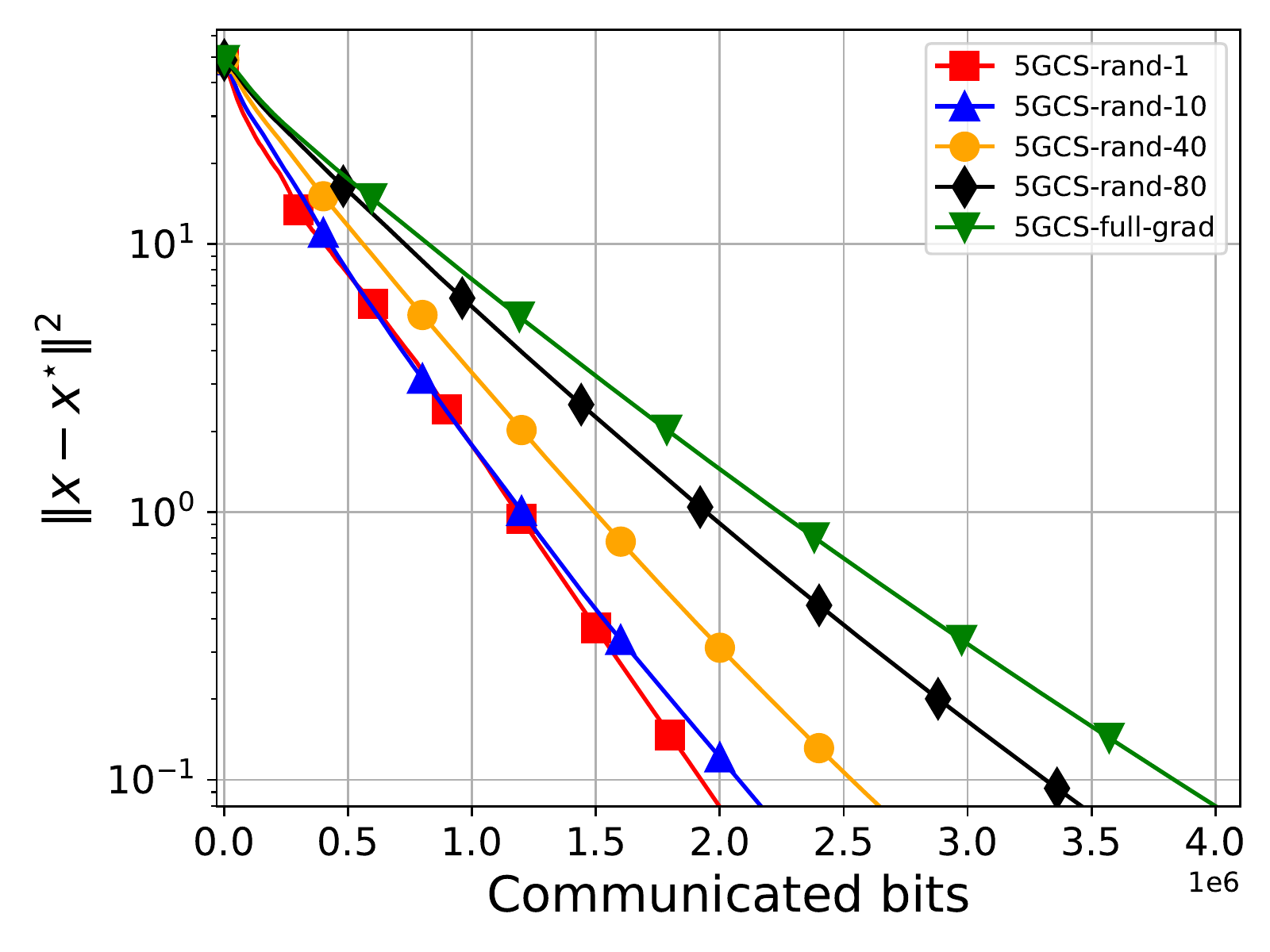}&\includegraphics[scale=0.415]{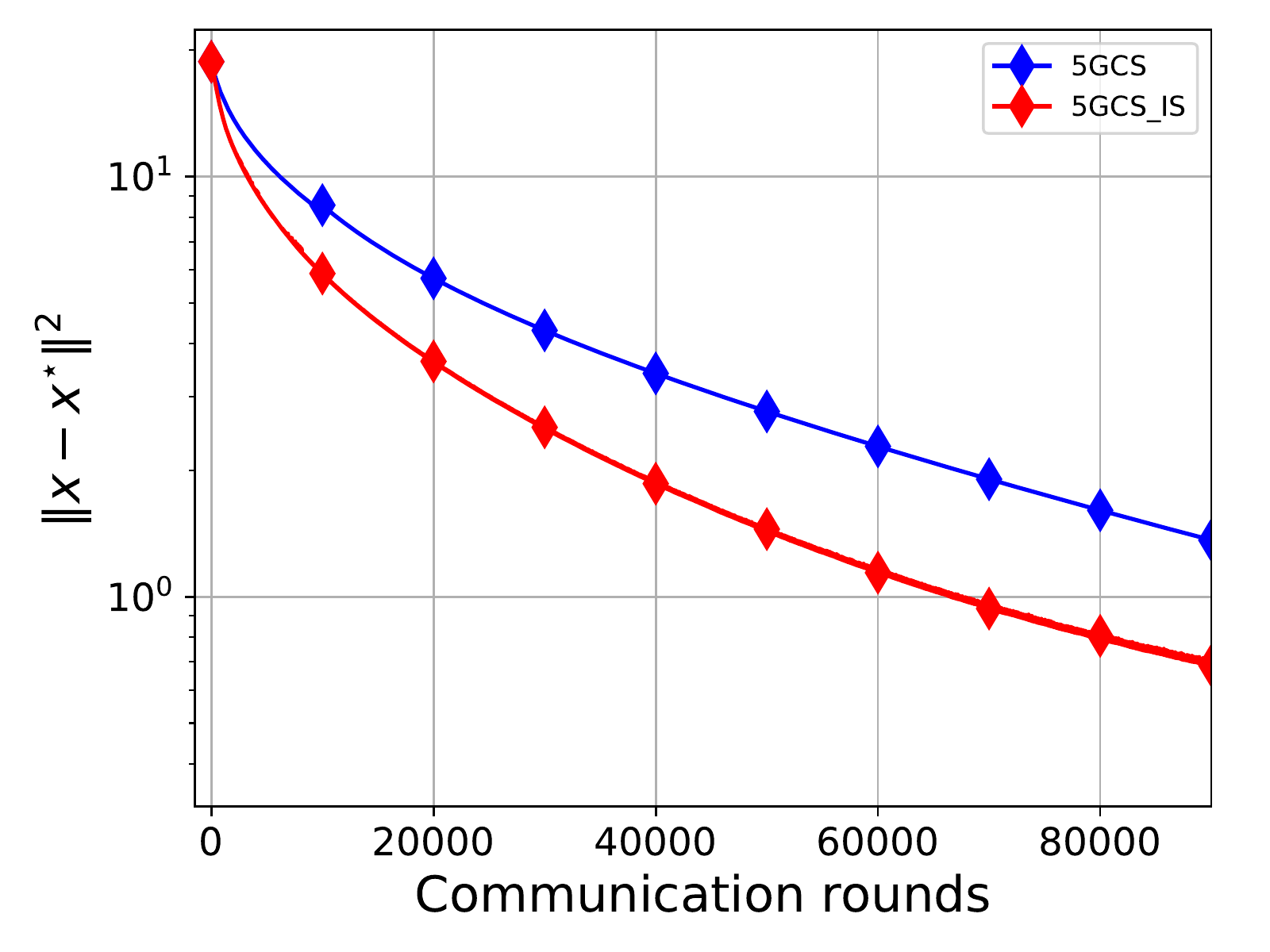}\\
		(a)&(b)
	\end{tabular}
	\caption{(a) Performance of Algorithm~\ref{alg:5GC-CC} (\algname{5GCS-CC}) with different levels of sparsification $k$. (b) Comparison of Algorithm~\ref{alg:5GCS-AB} (\algname{5GCS-AB}) with uniform sampling and Multisampling in case of $C=1$. }
	\label{fig1}
\end{figure*}

\subsection{Uniform Sampling}
In this section we show that previous result of ~\citet{grudzien2023can} can be covered by our framework. This means we fully generalize previous convergence guarantees. In this case we utilize uniform sampling without replacement.  
\begin{theorem}\label{thm:5GCSTN}
	Consider Algorithm~\ref{alg:5GCS-AB} with uniform sampling scheme satisfying~\ref{weighted} and LT solver satisfying Assumption~\ref{assm:localtr}. Let the inequality hold $\frac{1}{\tauM_\iclient}-\gammaM \Mx \geq \frac{4}{\tauM_\iclient^2}\frac{\mu}{3M}$, e.g. $\tauM_\iclient\geq\frac{8\mu}{3\Mx}$ and $\gammaM\leq\frac{1}{2\tauM_\iclient\Mx }$.  Then for the Lyapunov function 
	\begin{equation*}
		  	\Psi^{\kstep}\eqdef \frac{1}{\gammaM}\sqnorm{x^{t+1}-x^\star}+\sum_{\iclient=1}^\Mx\frac{\Mx}{\Cx}\left(\frac{1}{\tauM_\iclient}+\frac{1}{L_{F_\iclient}}\right)\sqnorm{u_\iclient^{t+1}-u_\iclient^\star},
	\end{equation*}
	the iterates of  the method satisfy
	\begin{equation*}
		 	\Exp{\Psi^{t+1}}\leq  \max\left\{\frac{1}{1+\gammaM\mu} ,\max_{m}\left[\frac{ L_{F_\iclient}+\frac{\Mx-\Cx}{\Mx}\tauM_\iclient}{L_{F_\iclient}+\tauM_\iclient}\right] \right\}\Exp{\Psi^{t}}.
	\end{equation*}
\end{theorem}

\begin{corollary}\label{cor:5GCSTN}
	Suppose that $L_m=L,\forall \iclient\in\left\{1,\ldots,\Mx\right\}$. Choose any $0<\varepsilon<1$ and $\gammaM = \frac{3}{16}\sqrt{\frac{\Cx}{L\mu \Mx}}$ and $\tauM_m=\frac{8}{3}\sqrt{\frac{L\mu}{\Mx\Cx}} $. In order to guarantee $\Exp{\Psi^{\Tx}}\leq \varepsilon \Psi^0$, it suffices to take 
	\begin{eqnarray}
		T &\geq	&  
		\max\left\{1+\frac{16}{3}\sqrt{\frac{\Mx}{\Cx}\frac{L}{\mu}},\frac{\Mx}{\Cx}+\frac{3}{8}\sqrt{\frac{\Mx}{\Cx}\frac{L}{\mu} }\right\}\log\frac{1}{\varepsilon} \notag \\
		&=&   
		\cO\left(\left(\frac{\Mx}{\Cx}+\sqrt{\frac{\Mx}{\Cx}\frac{L}{\mu }}\right) \log \frac{1}{\varepsilon} \right)\notag
	\end{eqnarray}
	communication rounds.
\end{corollary}

\section{Experiments}
This study primarily focuses on analyzing the fundamental algorithmic and theoretical aspects of a particular class of algorithms, rather than conducting extensive large-scale experiments. While we acknowledge the importance of such experiments, they fall outside the scope of this work. Instead, we provide illustrative examples and validate our findings through the application of logistic regression to a practical problem setting. 

We are considering $\ell_2$-regularized logistic regression, which is a mathematical model used for classification tasks. The objective function, denoted as $f(x)$, is defined as follows:
$$
 f(x) = \frac{1}{MN} \sum_{m=1}^M \sum_{i=1}^N \log \left(1 + e^{-b_{m, i} a_{m, i}^{\top} x}\right) + \frac{\lambda}{2}\|x\|^2.
$$
In this equation, $a_{m, i} \in \mathbb{R}^d$ and $b_{m, i} \in\{-1,+1\}$ represent the data samples and labels, respectively. The variables $M$ and $N$ correspond to the number of clients and the number of data points per client, respectively. The term $\lambda$ is a regularization parameter, and in accordance with \citet{condat2023tamuna}, we set $\lambda$, such that we have $\kappa = 10^4$. 

To illustrate our experimental results, we have chosen to focus on a specific case using the "a1a" dataset from the LibSVM library \citep{chang2011libsvm}. We have $d = 119$, $M=107$ and $N=15$ for this dataset.

For the experiments involving communication compression, we utilized the Rand-$k$ compressor \citep{DIANA} with various parameters for sparsification and theoretical stepsizes for the method. Based on the plotted results, it is evident that the optimal choice is achieved when setting $k=1$ and the method without communication compression shows the worst performance. We calculate the number of communicated floats by all clients.

In the experiments conducted to evaluate the Multisampling strategy, we employed the exact version of the parameters outlined in Corollary~\ref{cor:5GCS2}. Additionally, we applied a re-scaling procedure to modify the distribution of $L_m$ in order to reduce its uniformity. The resulting values were approximately $L_{\min} \approx 1.48$ and $L_{\max} \approx 2\cdot 10^4$.

The observed results indicate that the exact solution of determining probabilities and stepsizes., despite not being optimal, outperformed the version with uniform sampling.

\bibliography{5GCS-upd.bib}
\bibliographystyle{plainnat}

\clearpage
\appendix

\part*{Supplementary Materials}
\section{Basic Inequalities}

\subsection{Young's inequalities} For all $x,y\in \mathbb{R}^d$ and all $a>0$, we have
\begin{eqnarray}
	&&\langle x, y\rangle\leq \frac{a\sqnorm{x}}{2}+\frac{\sqnorm{y}}{2a},\label{yi1}\\
	&&\sqnorm{x+y}\leq2\sqnorm{x}+2\sqnorm{y}\label{yi2},\\
	&&\frac{1}{2}\sqnorm{x}-\sqnorm{y}\leq \sqnorm{x+y}.\label{yi3}
\end{eqnarray}

\subsection{Variance decomposition} For a random vector $\mathrm{X}\in\mathbb{R}^d$ (with finite second moment) and any $c\in\mathbb{R}^d$, the variance of $X$ can be decomposed as
\begin{eqnarray}
	\label{eq:var-dec}
	\Exp{\sqnorm{\mathrm{X}-\Exp{\mathrm{X}}}}=\Exp{\sqnorm{\mathrm{X}-c}}-\sqnorm{\Exp{\mathrm{X}}-c}.\label{vardec}
\end{eqnarray}
\subsection{Conic compression variance}	An unbiased randomized mapping  $\cC: \R^d\to \R^d$ has conic variance if there exists $\omega\geq 0$ such that
\begin{equation}
	\Exp{\sqnorm{\mathcal{C}(x)-x}}\leq \omega \sqnorm{x}\label{cvar}
\end{equation}
for all $x\in \R^d.$

\subsection{Convexity and $L$-smoothness}

Suppose $\phi\colon\mathbb{R}^d\to\mathbb{R}$ is $L$-smooth and convex. Then
\begin{equation}
	\frac{1}{L}\sqnorm{\nabla \phi(x)-\nabla \phi(y)}\leq \langle \nabla \phi(x)-\nabla \phi(y),x-y \rangle\label{strmono}
\end{equation}
for all $x,y\in\mathbb{R}^d$.
\subsection{Dual Problem and Saddle-Point Reformulation}

Then the saddle function reformulation of \myref{lifted_form} is:
\begin{equation}
	\mathrm{Find} \ (x^\star,(u_\iclient^\star)_{\iclient=1}^\Mx) \in \arg\min_{x\in\R^d}\max_{u\in\R^{\Mx d}} \, \left( \frac{\mu}{2}\sqnorm{x}+\sum_{\iclient=1}^\Mx \left\langle  x,u_\iclient \right\rangle -\sum_{\iclient=1}^\Mx F_\iclient^*(u_\iclient)\right).
	\label{saddlenew}
\end{equation}
To ensure well-posedness of these problems, we need to assume that there exists $x^\star\in\mathbb{R}^d$ s.t.:
\begin{align}
	0=\mu x^\star+\sum_{\iclient=1}^{\Mx} \nabla F_\iclient(x^\star).
\end{align}
Which is equivalent to \myref{lifted_form}, having a solution, which it does (unique in fact) as each $f_\iclient$ is $\mu$-strongly convex.
By first order optimality condition $x^\star$ and $u^\star$ that are solution to \myref{saddlenew}, satisfy:
\begin{equation}
	\left\{ \begin{array}{l}
		0= \mu x^\star + \sum_{\iclient=1}^\Mx u_\iclient^{\star}\\
		\Koper x^\star\in\partial F^* (u^{\star})
	\end{array}\right..\label{fooc}
\end{equation}
Where the latter in \myref{fooc} is equivalent to:
\begin{equation}
	\nabla F(\Koper x^\star)=u^\star.
\end{equation}
Throughout, this section we will denote by $\mathcal{F}_t$ for all $ t\geq 0$ the $\sigma$-algebra generated by the collection of $\left(\R^d\times\R^{d \Mx}\right)$-valued random variables $\left(x^0,u^0\right),\dots,\left(x^t,u^t\right).$

\section{Proof of Theorem \ref{thm:5GCS-AB}}

\begin{theorem*}\label{thm:5GCSMS}
	Let Assumption~\ref{assm:L-smooth-conv} hold. Consider Algorithm~\ref{alg:5GCS-AB} with sampling scheme $\mathbf{S}$ satisfying Assumption~\ref{weighted} and LT solvers $\mathcal{A}_m$ satisfying Assumption~\ref{assm:localtr}. Let the inequality hold $\frac{1}{\tauM_\iclient}-\left(\gammaM\left(1-B\right)\Mx +\gammaM\frac{A}{w_\iclient}\right)\geq \frac{4}{\tauM_\iclient^2}\frac{\mu_\iclient}{3M}$, e.g. $\tauM_\iclient\geq\frac{8\mu_\iclient}{3\Mx}$ and $\gammaM\leq\frac{1}{2\tauM_\iclient\left(\left(1-B\right)\Mx +\frac{A}{w_\iclient}\right)}$.  Then for the Lyapunov function 
	\begin{equation*}
		  	\Psi^{\kstep}\eqdef \frac{1}{\gammaM}\sqnorm{x^{t}-x^\star}+\sum_{\iclient=1}^\Mx\left(1+\qm\right)\left(\frac{1}{\tauM_\iclient}+\frac{1}{L_{F_\iclient}}\right)\sqnorm{u_\iclient^{t}-u_\iclient^\star},
	\end{equation*}
	the iterates of  the method satisfy
	\begin{equation*}
	 	\Exp{\Psi^{t+1}}\leq  \max\left\{\frac{1}{1+\gammaM\mu} ,\max_{m}\left[\frac{ L_{F_\iclient}+\frac{\qm}{1+\qm}\tauM_\iclient}{L_{F_\iclient}+\tauM_\iclient}\right] \right\}\Exp{\Psi^{t}},
	\end{equation*}
	where $\qm=\frac{1}{\widehat{p}_\iclient}-1$ and $\widehat{p}_\iclient$ is probability that $m$-th client is participating. 
\end{theorem*}

\begin{proof}
	We start from using variance decomposition \ref{eq:var-dec} and Proposition 1 from \citep{condat2021murana}, we obtain
	\begin{eqnarray}
		\Exp{\sqnorm{x^{t+1}-x^\star}\;|\;\mathcal{F}_t}&\overset{(\ref{vardec})}{=}&\sqnorm{\Exp{x^{t+1}\;|\;\mathcal{F}_t}-x^\star}+\Exp{\sqnorm{x^{t+1}-\Exp{x^{t+1}\;|\;\mathcal{F}_t}}\;|\;\mathcal{F}_t}\notag\\&\overset{(\ref{weighted})}{=}&   \notag\underbrace{ \sqnorm{\hat{x}^{t}-x^\star-\gammaM \Koper^\top\left(\bar{u}^{t+1}-u^t \right)}}_X-\gammaM^2B\sqnorm{\Koper^\top (\bar{u}^{t+1}-u^t)}\\
		&\quad&+\gammaM^2\sum_{\iclient=1}^{\Mx}\frac{A}{w_\iclient}\sqnorm{\bar{u}_\iclient^{t+1}-u_\iclient^t}.
		\label{x102}
	\end{eqnarray}
	Moreover, using \myref{fooc} and the definition of $\hat{x}^t$, we have
	\begin{align}
		&(1+\gammaM\mu)\hat{x}^t=x^t-\gammaM \Koper^\top u^{t},\label{opt13}\\
		&(1+\gammaM\mu)x^\star= x^\star -\gammaM\Koper^\top u^\star.\label{opt23}
	\end{align}
	Using \myref{opt13} and \myref{opt23} we obtain
	\begin{eqnarray}
		X&=&\sqnorm{\hat{x}^{t}-x^\star} +\gammaM^2\sqnorm{\Koper^\top\left(\bar{u}^{t+1}-u^t \right)}-2\gammaM \left\langle 
		\hat{x}^{t}-x^\star,\Koper^\top\left(\bar{u}^{t+1}-u^t \right) \right\rangle  \notag\\
		&\leq &(1+\gammaM\mu) \sqnorm{\hat{x}^{t}-x^\star} +\gammaM^2\sqnorm{\Koper^\top\left(\bar{u}^{t+1}-u^t \right)}\notag\\
		&\quad&-2\gammaM  \left\langle 
		\hat{x}^{t}-x^\star,\Koper^\top\left(\bar{u}^{t+1}-u^\star \right) \right\rangle  +2\gammaM  \left\langle 
		\hat{x}^{t}-x^\star,\Koper^\top\left(u^{t}-u^\star \right) \right\rangle  \notag\\
		&\overset{(\ref{opt13})+(\ref{opt23})}{=} &  \left\langle  x^t-x^\star-\gammaM \Koper^\top\left(u^{t}-u^\star \right),\hat{x}^{t}-x^\star \right\rangle  \notag+\gammaM^2\sqnorm{\Koper^\top\left(\bar{u}^{t+1}-u^t \right)}\notag\\
		&\quad&-2\gammaM  \left\langle 
		\hat{x}^{t}-x^\star,\Koper^\top\left(\bar{u}^{t+1}-u^\star \right) \right\rangle  +  \left\langle 
		\hat{x}^{t}-x^\star,2\gammaM\Koper^\top\left(u^{t}-u^\star \right)\right\rangle  \notag\\
		&= &\left\langle  x^t-x^\star+\gammaM \Koper^\top\left(u^{t}-u^\star \right),\hat{x}^{t}-x^\star \right\rangle  +\gammaM^2\sqnorm{\Koper^\top\left(\bar{u}^{t+1}-u^t \right)}\notag\\
		&\quad&-2\gammaM  \left\langle 
		\hat{x}^{t}-x^\star,\Koper^\top\left(\bar{u}^{t+1}-u^\star \right) \right\rangle  \notag\\
		&\overset{(\ref{opt13})+(\ref{opt23})}{=}&\frac{1}{1+\gammaM\mu} \left\langle  x^t-x^\star+\gammaM \Koper^\top\left(u^{t}-u^\star \right),x^t-x^\star-\gammaM \Koper^\top\left(u^{t}-u^\star \right) \right\rangle  \notag\\
		&\quad&+\gammaM^2\sqnorm{\Koper^\top\left(\bar{u}^{t+1}-u^t \right)}-2\gammaM  \left\langle 
		\hat{x}^{t}-x^\star,\Koper^\top\left(\bar{u}^{t+1}-u^\star \right) \right\rangle  \notag\\
		&= &\frac{1}{1+\gammaM\mu}\sqnorm{x^t-x^\star}-\frac{\gammaM^2}{1+\gammaM\mu}\sqnorm{\Koper^\top\left(u^{t}-u^\star \right)}\notag\\
		&\quad&+\gammaM^2\sqnorm{\Koper^\top\left(\bar{u}^{t+1}-u^t \right)}-2\gammaM  \left\langle 
		\hat{x}^{t}-x^\star,\Koper^\top\left(\bar{u}^{t+1}-u^\star \right) \right\rangle  . \label{long31} 
	\end{eqnarray}
	Combining \myref{x102} and \myref{long31}
	\begin{eqnarray*}
		\Exp{\sqnorm{x^{t+1}-x^\star}\;|\;\mathcal{F}_t}&\leq&  \frac{1}{1+\gammaM\mu}\sqnorm{x^t-x^\star}-\frac{\gammaM^2}{1+\gammaM\mu}\sqnorm{\Koper^\top (u^t-u^\star)}\\
		&\quad&+\gammaM^2(1-B)\sqnorm{\Koper^\top (\bar{u}^{t+1}-u^t)}\\
		&\quad&-2\gammaM  \left\langle 
		\hat{x}^{t}-x^\star,\Koper^\top (\bar{u}^{t+1}-u^\star) \right\rangle \\
		&\quad&+\gammaM^2\sum_{\iclient=1}^{\Mx}\frac{A}{w_\iclient}\sqnorm{\bar{u}_\iclient^{t+1}-u_\iclient^t}-\frac{\gammaM\mu}{\Mx}\sqnorm{\Koper\hat{x}^t-\Koper x^\star}.
	\end{eqnarray*}

Let $\widehat{p}=\left(\widehat{p}_1,\ldots,\widehat{p}_\Mx\right)$. The update for $u$ may be written as $$u_\iclient^{t+1}=u_\iclient^t+   \widehat{p}_m\frac{1}{\widehat{p}_m} Bernoulli(\bar{u}_\iclient^{t+1}-u^t_\iclient,\widehat{p}_m),$$ where $\widehat{p}_\iclient$ is the probability that client $\iclient$ participates in the iteration. Firstly note that the update for $u_m^{t+1}$ can be written as:
$$u_\iclient^{t+1}=u_\iclient^t+\widehat{p}_m\Ropp_\iclient(\bar{u}_\iclient^{t+1}-u^t_\iclient,\widehat{p}_m),$$ 
i.e
we have a relation of $\frac{1}{1+\qm}=\widehat{p}_m$ , which obviously makes sense, since the independent, unbiased bernoulli compressor with probability $\prm$ has conic variance $\qm=\frac{1}{\widehat{p}_m}-1$. This leads to
	$$u_\iclient^{t+1}=u_\iclient^t+ \frac{1}{1+\qm}\Ropp_\iclient(\bar{u}_\iclient^{t+1}-u^t_\iclient,\qm).$$ 
	Using such form, we get
	\begin{eqnarray}
		\notag \Exp{\sqnorm{u_\iclient^{t+1}-u_\iclient^\star}\;|\;\mathcal{F}_t} &\overset{(\ref{vardec})+(\ref{cvar})}{\leq}& \sqnorm{u_\iclient^{t}-u_\iclient^\star+\frac{1}{1+\qm}\left(\bar{u}_\iclient^{t+1} -u_\iclient^t\right)}\\
		&&+\frac{\qm}{(1+\qm)^2}\sqnorm{\bar{u}_\iclient^{t+1} -u_\iclient^t }\notag\\
		&=&\frac{\qm^2}{(1+\qm)^2}\sqnorm{u_\iclient^{t}-u_\iclient^\star}+\frac{1}{(1+\qm)^2}\sqnorm{\bar{u}_\iclient^{t+1}-u_\iclient^\star}\notag\\
	\notag	&\quad&+\frac{2\qm}{(1+\qm)^2} \left\langle  u_\iclient^{t}-u_\iclient^\star,
		\bar{u}_\iclient^{t+1}-u_\iclient^\star \right\rangle\\\notag &&+\frac{\qm}{(1+\qm)^2}\sqnorm{\bar{u}_\iclient^{t+1} -u_\iclient^\star }\notag+\frac{\qm}{(1+\qm)^2}\sqnorm{u_\iclient^{t} -u_\iclient^\star }\\
		&&\notag-\frac{2\qm}{(1+\qm)^2} \left\langle  u_\iclient^{t}-u_\iclient^\star,
		\bar{u}_\iclient^{t+1}-u_\iclient^\star \right\rangle \notag\\
		&\leq&\frac{1}{1+\qm}\sqnorm{\bar{u}_\iclient^{t+1} -u_\iclient^\star }+\frac{\qm}{1+\qm}\sqnorm{u_\iclient^{t} -u_\iclient^\star}.\label{newubound2}
	\end{eqnarray}
	Let us consider the first term in \myref{newubound2}:
	\begin{eqnarray*}
		\sqnorm{\bar{u}_\iclient^{t+1}-u_\iclient^\star}&=&\sqnorm{(u_\iclient^t-u_\iclient^\star)+(\bar{u}_\iclient^{t+1}-u_\iclient^t)}\\
		&=&\sqnorm{u_\iclient^t-u_\iclient^\star}+\sqnorm{\bar{u}_\iclient^{t+1}-u_\iclient^t}+2 \left\langle  u_\iclient^t-u_\iclient^\star,\bar{u}_\iclient^{t+1}-u_\iclient^t \right\rangle \\
		&=&\sqnorm{u_\iclient^t-u_\iclient^\star}+2  \left\langle  \bar{u}_\iclient^{t+1}-u_\iclient^\star,\bar{u}_\iclient^{t+1}-u_\iclient^t \right\rangle  - \sqnorm{\bar{u}_\iclient^{t+1}-u_\iclient^t}.
	\end{eqnarray*}
	Combining terms together we get
	\begin{align*}
		\Exp{\sqnorm{u_\iclient^{t+1}-u_\iclient^\star}\;|\;\mathcal{F}_t}&\leq \sqnorm{u_\iclient^{t} -u_\iclient^\star }\\
		&+\frac{1}{1+\qm}\left(2  \left\langle  \bar{u}_\iclient^{t+1}-u_\iclient^\star,\bar{u}_\iclient^{t+1}-u_\iclient^t \right\rangle  - \sqnorm{\bar{u}_\iclient^{t+1}-u_\iclient^t}\right).
	\end{align*}
	Finally, we obtain
	\begin{eqnarray*}
		\frac{1}{\gammaM}\Exp{\sqnorm{x^{t+1}-x^\star}\;|\;\mathcal{F}_t}&+&\sum_{\iclient=1}^\Mx\frac{1+\qm}{\tauM_\iclient}\Exp{\sqnorm{u_\iclient^{t+1}-u_\iclient^\star}\;|\;\mathcal{F}_t}\\
		&\leq&  \frac{1}{\gammaM\left(1+\gammaM\mu\right)}\sqnorm{x^t-x^\star}-\frac{\gammaM}{1+\gammaM\mu}\sqnorm{\Koper^\top (u^t-u^\star)}\\
		&&+\gammaM(1-B)\sqnorm{\Koper^\top (\bar{u}^{t+1}-u^t)}\\
		&&+\gammaM\sum_{\iclient=1}^{\Mx}\frac{A}{w_\iclient}\sqnorm{\bar{u}_\iclient^{t+1}-u_\iclient^t}-\frac{\mu}{\Mx}\sqnorm{\Koper\hat{x}^t-\Koper x^\star}\\
		&&+\frac{1+\qm}{\tauM_\iclient}\sqnorm{u_\iclient^{t} -u_\iclient^\star }-2 \sum_{\iclient=1}^\Mx \left\langle 
		\hat{x}^{t}-x^\star,\bar{u}_\iclient^{t+1}-u_\iclient^\star \right\rangle \\ 
		&&+\frac{1}{\tauM_\iclient}\left(2  \left\langle  \bar{u}_\iclient^{t+1}-u_\iclient^\star,\bar{u}_\iclient^{t+1}-u_\iclient^t \right\rangle  - \sqnorm{\bar{u}_\iclient^{t+1}-u_\iclient^t}\right).
	\end{eqnarray*}
	Ignoring $-\frac{\gammaM}{1+\gammaM\mu}\sqnorm{\Koper^\top (u^t-u^\star)}$ and noting
	\begin{align*}
		&-\left\langle \hat{x}^{t}-x^\star,\bar{u}_\iclient^{t+1}-u_\iclient^\star \right\rangle +\frac{1}{\tauM_\iclient}\left\langle  \bar{u}_\iclient^{t+1}-u_\iclient^\star,\bar{u}_\iclient^{t+1}-u_\iclient^t \right\rangle  \\
		&=-\left\langle \lastlocittermk- x^\star,\bar{u}_\iclient^{t+1}-u_\iclient^\star \right\rangle +\frac{1}{\tauM_\iclient}\left\langle  \nabla\localfuni (\lastlocittermk),\bar{u}_\iclient^{t+1}-u_\iclient^\star \right\rangle\\
		&\overset{(\ref{yi1})+(\ref{strmono})}{\leq} -\frac{1}{L_{F_\iclient}}\sqnorm{\bar{u}_\iclient^{t+1}-u_\iclient^\star}+\frac{a_\iclient}{2\tauM_\iclient}\sqnorm{\nabla\localfuni (\lastlocittermk)}+\frac{1}{2a_\iclient\tauM_\iclient}\sqnorm{\bar{u}_\iclient^{t+1}-u_\iclient^\star }\\
		&= -\left(\frac{1}{L_{F_\iclient}}-\frac{1}{2a_\iclient\tauM_\iclient}\right)\sqnorm{\bar{u}_\iclient^{t+1}-u_\iclient^\star}+\frac{a_\iclient}{2\tauM_\iclient}\sqnorm{\nabla\localfuni (\lastlocittermk)}\\
		&\overset{(\ref{newubound2})}{\leq} -\left(\frac{1}{L_{F_\iclient}}-\frac{1}{2a_\iclient\tauM_\iclient}\right)\left(\left(1+\qm\right)\Exp{\sqnorm{u_\iclient^{t+1}-u_\iclient^\star}\;|\;\mathcal{F}_t}-\qm\sqnorm{u_\iclient^t-u_\iclient^\star}\right)\\
		&+\frac{a_\iclient}{2\tauM_\iclient}\sqnorm{\nabla\localfuni (\lastlocittermk)},
	\end{align*}
	we get
	\begin{eqnarray*}
		\frac{1}{\gammaM}\Exp{\sqnorm{x^{t+1}-x^\star}\;|\;\mathcal{F}_t}&+&\sum_{\iclient=1}^\Mx\left(1+\qm\right)\left(\frac{1}{\tauM_\iclient}+\frac{1}{L_{F_\iclient}}\right)\Exp{\sqnorm{u_\iclient^{t+1}-u_\iclient^\star}\;|\;\mathcal{F}_t}\\
		&\leq&  \frac{1}{\gammaM\left(1+\gammaM\mu\right)}\sqnorm{x^t-x^\star}\\
		&+&\sum_{\iclient=1}^\Mx \left(1+\qm\right)\left(\frac{1}{\tauM_\iclient}+\frac{\qm}{1+\qm}\frac{1}{L_{F_\iclient}}\right)\sqnorm{u_\iclient^{t} -u_\iclient^\star }\\
		&&+\sum_{\iclient=1}^\Mx\left(\gammaM\left(1-B\right)\Mx +\gammaM\frac{A}{w_\iclient}-\frac{1}{\tauM_\iclient}\right)\sqnorm{\bar{u}_\iclient^{t+1}-u_\iclient^t}\\
		&&+\sum_{\iclient=1}^\Mx\frac{L_{F_\iclient}}{\tauM_\iclient^2}\sqnorm{\nabla\localfuni (\lastlocittermk)}-\sum_{\iclient=1}^\Mx\mu v_\iclient\sqnorm{\hat{x}^t- x^\star}.
	\end{eqnarray*}
	Where we made the choice $a_\iclient=\frac{L_{F_\iclient}}{\tau_\iclient}$
	and $\sum_{\iclient=1}^\Mx v_\iclient\leq 1$, positive real numbers,  e.g. $\frac{\mu_m}{\sum_{m=1}^{M} \mu_m }$.
	Using Young's inequality we have
	\begin{equation*}
		-\frac{\mu v_\iclient}{3}\sqnorm{ \hat{x}^t-\localsolmk+\localsolmk- x^\star}\overset{(\ref{yi3})}{\leq}\frac{\mu v_\iclient}{3}\sqnorm{\localsolmk- x^\star}-\frac{\mu v_\iclient}{6}\sqnorm{\hat{x}^t-\localsolmk}.
	\end{equation*}
	Noting the fact that $\localsolmk=\hat{x}^t-\frac{1}{\tauM_\iclient}(\hat{u}_\iclient^{t+1}-u_\iclient^t)$, we have
	$$\frac{\mu v_\iclient}{3}\sqnorm{\localsolmk- x^\star}\overset{(\ref{yi2})}{\leq} 2\frac{\mu v_\iclient}{3}\sqnorm{\hat{x}^t-x^\star}+\frac{2}{\tauM_\iclient^2}\frac{\mu v_\iclient}{3}\sqnorm{\hat{u}_\iclient^{t+1}-u_\iclient^t}.$$ 
	Combining those inequalities we get
	\begin{eqnarray*}
		\frac{1}{\gammaM}\Exp{\sqnorm{x^{t+1}-x^\star}\;|\;\mathcal{F}_t}&+&\sum_{\iclient=1}^\Mx\left(1+\qm\right)\left(\frac{1}{\tauM_\iclient}+\frac{1}{L_{F_\iclient}}\right)\Exp{\sqnorm{u_\iclient^{t+1}-u_\iclient^\star}\;|\;\mathcal{F}_t}\\
		&\leq&  \frac{1}{\gammaM\left(1+\gammaM\mu\right)}\sqnorm{x^t-x^\star}\\
		&&+\sum_{\iclient=1}^\Mx \left(1+\qm\right)\left(\frac{1}{\tauM_\iclient}+\frac{\qm}{1+\qm}\frac{1}{L_{F_\iclient}}\right)\sqnorm{u_\iclient^{t} -u_\iclient^\star }\\
		&&+\sum_{\iclient=1}^\Mx \frac{2}{\tauM_\iclient^2}\frac{\mu v_\iclient}{3}\sqnorm{\hat{u}_\iclient^{t+1}-u_\iclient^t}\\
		&&-\sum_{\iclient=1}^\Mx\left(\frac{1}{\tauM_\iclient}-\left(\gammaM\left(1-B\right)\Mx +\gammaM\frac{A}{w_\iclient}\right)\right)\sqnorm{\bar{u}_\iclient^{t+1}-u_\iclient^t}\\
		&&+\sum_{\iclient=1}^\Mx\frac{L_{F_\iclient}}{\tauM_\iclient^2}\sqnorm{\nabla\localfuni (\lastlocittermk)}-\sum_{\iclient=1}^\Mx\frac{\mu v_\iclient}{6}\sqnorm{\hat{x}^t-\localsolmk}.
	\end{eqnarray*}
	Assuming $\gammaM$ and $\tauM_\iclient$ can be chosen so that $\frac{1}{\tauM_\iclient}-\left(\gammaM\left(1-B\right)\Mx +\gammaM\frac{A}{w_\iclient}\right)\geq \frac{4}{\tauM_\iclient^2}\frac{\mu v_\iclient}{3}$ we obtain
	\begin{eqnarray*}
		\frac{1}{\gammaM}\Exp{\sqnorm{x^{t+1}-x^\star}\;|\;\mathcal{F}_t}&+&\sum_{\iclient=1}^\Mx\left(1+\qm\right)\left(\frac{1}{\tauM_\iclient}+\frac{1}{L_{F_\iclient}}\right)\Exp{\sqnorm{u_\iclient^{t+1}-u_\iclient^\star}\;|\;\mathcal{F}_t}\\
		&\leq&  \frac{1}{\gammaM\left(1+\gammaM\mu\right)}\sqnorm{x^t-x^\star}\\
		&\quad&+\sum_{\iclient=1}^\Mx \left(1+\qm\right)\left(\frac{1}{\tauM_\iclient}+\frac{\qm}{1+\qm}\frac{1}{L_{F_\iclient}}\right)\sqnorm{u_\iclient^{t} -u_\iclient^\star }\\
		&\quad&+\sum_{\iclient=1}^\Mx \frac{4}{\tauM_\iclient^2}\frac{\mu v_\iclient L_{F_\iclient}^2}{3}\sqnorm{\lastlocittermk-\localsolmk}\\
		&\quad&+\sum_{\iclient=1}^\Mx\frac{L_{F_\iclient}}{\tauM_\iclient^2}\sqnorm{\nabla\localfuni (\lastlocittermk)}-\sum_{\iclient=1}^\Mx\frac{\mu v_\iclient}{6}\sqnorm{\hat{x}^t-\localsolmk}.
	\end{eqnarray*}
	 The point $\lastlocitterk$ is supposed to satisfy Assumption~\ref{assm:localtr}:
	\begin{eqnarray*}
		\sum_{\iclient=1}^\Mx\frac{4}{\tauM_\iclient^2}\frac{\mu v_\iclient L^2_{F_\iclient}}{3}\sqnorm{\lastlocittermk-\localsolmk}+\sum_{\iclient=1}^\Mx\frac{L_{F_\iclient}}{\tauM_\iclient^2}\sqnorm{\nabla\localfuni (\lastlocittermk)}\leq\sum_{\iclient=1}^\Mx\frac{\mu v_\iclient}{6}\sqnorm{\hat{x}^t-\localsolmk}.
	\end{eqnarray*}
	Thus
	\begin{eqnarray*}
		\frac{1}{\gammaM}\Exp{\sqnorm{x^{t+1}-x^\star}\;|\;\mathcal{F}_t}&+&\sum_{\iclient=1}^\Mx\left(1+\qm\right)\left(\frac{1}{\tauM_\iclient}+\frac{1}{L_{F_\iclient}}\right)\Exp{\sqnorm{u_\iclient^{t+1}-u_\iclient^\star}\;|\;\mathcal{F}_t}\\
		&\leq&  \frac{1}{\gammaM\left(1+\gammaM\mu\right)}\sqnorm{x^t-x^\star}\\
		&\quad&+\sum_{\iclient=1}^\Mx \left(1+\qm\right)\left(\frac{1}{\tauM_\iclient}+\frac{\qm}{1+\qm}\frac{1}{L_{F_\iclient}}\right)\sqnorm{u_\iclient^{t} -u_\iclient^\star }.
	\end{eqnarray*}
	By taking the expectation on both sides we get
	\begin{equation*}
		\Exp{\Psi^{t+1}}\leq  \max\left\{\frac{1}{1+\gammaM\mu} ,\frac{ L_{F_\iclient}+\frac{\qm}{1+\qm}\tauM_\iclient}{L_{F_\iclient}+\tauM_\iclient} \right\}\Exp{\Psi^{t}},
	\end{equation*}
	which finishes the proof.
\end{proof}
\section{Multisampling (Sampling with Replacement)}
\subsection{Proof of Lemma~\ref{lem:multi}}
\begin{lemma*}
	The Multisampling with estimator~\ref{eq:swr} satisfies the Assumption~\ref{weighted} with $A= B = \frac{1}{C}$ and $w_m=\prm$.
\end{lemma*}

\begin{proof} The proof is presented in \citet{PAGE-AS}. Let us provide it for completeness. 
	 
Let us fix $\Cx>0$. For all $\iclient\in[\Cx]$, we define i.i.d. random variables
\begin{eqnarray*}
	\mathcal{X}_\iclient=\begin{cases}
		1 & \text{with probability }p_1\\
		2 & \text{with probability }p_2\\
		& \cdot\\
		& \cdot\\
		& \cdot\\
		\Mx & \text{with probability }p_\Mx,
	\end{cases}
\end{eqnarray*}
where $\underline{p}=\left(p_1, \ldots, p_M\right) \in \Delta^M$(simple simplex). A sampling 
\begin{equation*}
	S(a_1,\ldots,a_M;\underline{p})\eqdef \frac{1}{\Cx}\sum_{\iclient=1}^{\Cx}\frac{a_{\mathcal{X}_\iclient}}{\Mx p_{\mathcal{X}_\iclient}}
\end{equation*}
is called the Importance sampling. 

Let us establish inequality for Assumption~\ref{weighted}: 
\begin{eqnarray*}
&&	\Exp{\sqnorm{\frac{1}{\Cx}\sum_{\iclient=1}^{\Cx}\frac{a_{\mathcal{X}_\iclient}}{\Mx p_{\mathcal{X}_\iclient}}-\frac{1}{\Mx}\sum_{\iclient=1}^{\Mx}a_\iclient}}=\frac{1}{\Cx^2}\sum_{\iclient=1}^{\Cx}\Exp{\sqnorm{\frac{a_{\mathcal{X}_\iclient}}{\Mx p_{\mathcal{X}_\iclient}}-\frac{1}{\Mx}\sum_{\iclient=1}^{\Mx}a_\iclient}}\\
	&\quad&+\frac{1}{\Cx^2}\sum_{\iclient\neq\iclient'}\Exp{\left\langle\frac{a_{\mathcal{X}_\iclient}}{\Mx p_{\mathcal{X}_\iclient}}-\frac{1}{\Mx}\sum_{\iclient=1}^{\Mx}a_\iclient,\frac{a_{\mathcal{X}_\iclient'}}{\Mx p_{\mathcal{X}_\iclient'}}-\frac{1}{\Mx}\sum_{\iclient=1}^{\Mx}a_\iclient\right\rangle}.
\end{eqnarray*}
By utilizing the independence and unbiasedness of the random variables, the final term becomes zero, resulting in:
\begin{eqnarray*}
	\Exp{\sqnorm{\frac{1}{\Cx}\sum_{\iclient=1}^{\Cx}\frac{a_{\mathcal{X}_\iclient}}{\Mx p_{\mathcal{X}_\iclient}}-\frac{1}{\Mx}\sum_{\iclient=1}^{\Mx}a_\iclient}}&=&\frac{1}{\Cx^2}\sum_{\iclient=1}^{\Cx}\Exp{\sqnorm{\frac{a_{\mathcal{X}_\iclient}}{\Mx p_{\mathcal{X}_\iclient}}-\frac{1}{\Mx}\sum_{\iclient=1}^{\Mx}a_\iclient}}\\
	&=&\frac{1}{\Cx^2}\sum_{\iclient=1}^{\Cx}\Exp{\sqnorm{\frac{a_{\mathcal{X}_\iclient}}{\Mx p_{\mathcal{X}_\iclient}}}}-\frac{1}{\Cx}\sqnorm{\frac{1}{\Mx}\sum_{\iclient=1}^{\Mx}a_\iclient}\\
	&=&\frac{1}{\Cx}\sum_{\iclient=1}^{\Mx}p_\iclient\sqnorm{\frac{a_{\iclient}}{\Mx p_\iclient}}-\frac{1}{\Cx}\sqnorm{\frac{1}{\Mx}\sum_{\iclient=1}^{\Mx}a_\iclient}\\
	&=&\frac{1}{\Cx}\left(\frac{1}{\Mx}\sum_{\iclient=1}^{\Mx}\frac{1}{\Mx p_\iclient}\sqnorm{a_{\iclient}}-\sqnorm{\frac{1}{\Mx}\sum_{\iclient=1}^{\Mx}a_\iclient}\right).
\end{eqnarray*}
Thus we have $A=B=\frac{1}{\Cx}$.
\end{proof}
\subsection{Proof of Theorem \ref{thm:5GCMult}}
\begin{theorem*} 
	Let Assumption~\ref{assm:L-smooth-conv} hold. Consider Algorithm~\ref{alg:5GCS-AB} (\algname{5GCS-AB}) with Multisampling and estimator~\ref{eq:swr} satisfying Assumption~\ref{weighted} and LT solvers $\mathcal{A}_m$ satisfying Assumption~\ref{assm:localtr}. Let the inequality hold $\frac{1}{\tauM_\iclient}-\left(\gammaM\left(1-\frac{1}{\Cx}\right)\Mx +\gammaM\frac{1}{\Cx \prm}\right)\geq \frac{4}{\tauM_\iclient^2}\frac{\mu_\iclient}{3\Mx}$ , e.g. $\tauM_\iclient\geq\frac{8\mu_\iclient}{3\Mx}$ and $\gammaM\leq\frac{1}{2\tauM_\iclient\left(\left(1-\frac{1}{\Cx}\right)\Mx +\frac{1}{\Cx \prm}\right)}$.  Then for the Lyapunov function 
	\begin{equation*}
		  	\Psi^{\kstep}\eqdef \frac{1}{\gammaM}\sqnorm{x^{t}-x^\star}+\sum_{\iclient=1}^\Mx\frac{1}{\widehat{p}_\iclient}\left(\frac{1}{\tauM_\iclient}+\frac{1}{L_{F_\iclient}}\right)\sqnorm{u_\iclient^{t}-u_\iclient^\star},
	\end{equation*}
	the iterates of  the method satisfy
	\begin{equation*}
	 	\Exp{\Psi^{t+1}}\leq  \max\left\{\frac{1}{1+\gammaM\mu} , \max_{m}\left[\frac{ L_{F_\iclient}+\left(1-\widehat{p}_\iclient\right)\tauM_\iclient}{L_{F_\iclient}+\tauM_\iclient} \right\}\right]\Exp{\Psi^{t}},
	\end{equation*}
	where $\widehat{p}_\iclient = 1-\left(1-\prm\right)^\Cx$ is probability that $m$-th client is participating. 
	
\end{theorem*}
\begin{proof}
	We start from theorem~\ref{thm:5GCS-AB}:
		\begin{equation*}
		\Exp{\Psi^{t+1}}\leq  \max\left\{\frac{1}{1+\gammaM\mu} ,\frac{ L_{F_\iclient}+\frac{\qm}{1+\qm}\tauM_\iclient}{L_{F_\iclient}+\tauM_\iclient} \right\}\Exp{\Psi^{t}}. 
	\end{equation*}
For Multisampling the probability of $m$-th client participating is $\widehat{p}_\iclient = 1-\left(1-\prm\right)^\Cx$ and we have relation $\widehat{p}_\iclient  = \frac{1}{1+q_m} $. Plugging $q_m = \frac{1}{\widehat{p}_m} - 1$ into recursion gives us 

	\begin{equation*}
	\Exp{\Psi^{t+1}}\leq  \max\left\{\frac{1}{1+\gammaM\mu} , \max_{m}\left[\frac{ L_{F_\iclient}+\left(1-\widehat{p}_\iclient\right)\tauM_\iclient}{L_{F_\iclient}+\tauM_\iclient} \right\}\right]\Exp{\Psi^{t}}.
\end{equation*}
Also using Lemma~\ref{lem:multi} we have $A = B = \frac{1}{C}$ and $w = p_m$. Plugging such constants to inequality for $\gamma$ and $\tau_m$ leads to $\frac{1}{\tauM_\iclient}-\left(\gammaM\left(1-\frac{1}{\Cx}\right)\Mx +\gammaM\frac{1}{\Cx \prm}\right)\geq \frac{4}{\tauM_\iclient^2}\frac{\mu_\iclient}{3\Mx}$ , e.g. $\tauM_\iclient\geq\frac{8\mu_\iclient}{3\Mx}$ and $\gammaM\leq\frac{1}{2\tauM_\iclient\left(\left(1-\frac{1}{\Cx}\right)\Mx +\frac{1}{\Cx \prm}\right)}$.  
\end{proof}
\subsection{Proof of Corollary~\ref{cor:5GCS1}}

\begin{corollary*}
	Suppose $\Cx=1$. Choose any $0<\varepsilon<1$ and $p_\iclient=\frac{\sqrt{L_{F,\iclient}+\tauM_\iclient}}{\sum_{\iclient=1}^\Mx \sqrt{L_{F,\iclient}+\tauM_\iclient}}$, and $\tauM_\iclient=\frac{8}{3}\sqrt{\overline{L}\mu\Mx}p_\iclient$.  In order to guarantee $\Exp{\Psi^{\Tx}}\leq \varepsilon \Psi^0$, it suffices to take 
	\begin{eqnarray*}
		  	T \geq	 
		\max\left\{1+\frac{16}{3}\sqrt{\frac{\overline{L}\Mx}{\mu}},\frac{3}{8}\sqrt{\frac{\overline{L}\Mx}{\mu}}+\Mx\right\} \log \frac{1}{\varepsilon}
	\end{eqnarray*}
	communication rounds.
\end{corollary*}

\begin{proof}
We set parameters as
$p_\iclient=\frac{\sqrt{L_{F_\iclient}+\tauM_\iclient}}{\sum_{\iclient=1}^\Mx \sqrt{L_{F_\iclient}+\tauM_\iclient}}$, and $\tauM_\iclient=\frac{8}{3}\sqrt{\overline{L}\mu\Mx}p_\iclient$. Let us derive the communication complexity:
\begin{align*}
	T&\geq\max\left\{1+\frac{16}{3}\sqrt{\frac{\overline{L}\Mx}{\mu}},\frac{3}{8}\sqrt{\frac{\overline{L}\Mx}{\mu}}+\Mx\right\}\log \frac{1}{\varepsilon}\\
	&\geq \max\left\{1+\frac{16}{3}\sqrt{\frac{\overline{L}\Mx}{\mu}},\frac{3}{8}\frac{\Mx \overline{L}+\Mx \frac{8}{3}\sqrt{\overline{L}\mu\Mx}}{\sqrt{\overline{L}\mu\Mx}}\right\}\log \frac{1}{\varepsilon}\\
	&\geq\max\left\{1+\frac{16}{3}\sqrt{\frac{\overline{L}\Mx}{\mu}},\frac{3}{8}\frac{\left(\sum_{\iclient=1}^\Mx \sqrt{L_{F_\iclient+\tauM_\iclient}}\right)^2}{\sqrt{\overline{L}\mu\Mx}}\right\}\log \frac{1}{\varepsilon}\\
	&\geq \max\left\{1+\frac{16}{3}\sqrt{\frac{\overline{L}\Mx}{\mu}},\max_{m}\left(\frac{3}{8}\frac{\left(\sum_{\iclient=1}^\Mx \sqrt{L_{F_\iclient+\tauM_\iclient}}\right)^2}{\left(L_{F_\iclient}+\tauM_\iclient\right)\sqrt{\overline{L}\mu\Mx}}\left(L_{F_\iclient}+\tauM_\iclient\right)\right)\right\}\log \frac{1}{\varepsilon}.
\end{align*}
Unrolling the recursion from Theorem~\ref{thm:5GCMult} we get 
	\begin{equation}
		\label{eq:5.4}
	\Exp{\Psi^{T}}\leq \left( \max\left\{\frac{1}{1+\gammaM\mu} , \max_{m}\left[\frac{ L_{F_\iclient}+\left(1-\widehat{p}_\iclient\right)\tauM_\iclient}{L_{F_\iclient}+\tauM_\iclient} \right\}\right]\right)^T\Psi^{0}.
\end{equation}
Using Lemma from \citet{malinovsky2021random} for recursion (Appendix B), we can state that derived $T$ is sufficient to guarantee \ref{eq:5.4}.

\subsection{Proof of Corollary~\ref{cor:5GCS2}}

\begin{corollary*}
	Suppose $C=1$. Choose any $0<\varepsilon<1$ and $p_\iclient=\frac{\sqrt{\frac{L_\iclient}{\Mx}}}{\sum_{\iclient=1}^\Mx \sqrt{\frac{L_\iclient}{\Mx}}}$, and $\tauM_\iclient=\frac{8}{3}\sqrt{\overline{L}\mu\Mx}p_\iclient$.  In order to guarantee $\Exp{\Psi^{\Tx}}\leq \varepsilon \Psi^0$, it suffices to take 
	\begin{eqnarray*}		  
		T \geq	
		\max\left\{1+\frac{16}{3}\sqrt{\frac{\overline{L}\Mx}{\mu}},\frac{3}{8}\sqrt{\frac{\overline{L}\Mx}{\mu}}+\frac{\sum_{\iclient=1}^\Mx \sqrt{L_\iclient}}{\sqrt{L_{\min}}}\right\}\log \frac{1}{\varepsilon}
	\end{eqnarray*}
	communication rounds. Note that $L_{\min} = \min_m L_m$.
\end{corollary*}
We set parameters as $p_\iclient=\frac{\sqrt{\frac{L_\iclient}{\Mx}}}{\sum_{\iclient=1}^\Mx \sqrt{\frac{L_\iclient}{\Mx}}}$, and $\tauM_\iclient=\frac{8}{3}\sqrt{\overline{L}\mu\Mx}p_\iclient$. Let us derive the communication complexity. Since $\left(\sum_{\iclient=1}^\Mx \sqrt{\frac{L_\iclient}{\Mx}}\right)^2\leq \Mx \overline{L}$ we have
\begin{align*}
	T &\geq\max\left\{1+\frac{16}{3}\sqrt{\frac{\overline{L}\Mx}{\mu}},\frac{3}{8}\sqrt{\frac{\overline{L}\Mx}{\mu}}+\frac{\sum_{\iclient=1}^\Mx \sqrt{L_\iclient}}{\sqrt{L_{\min}}}\right\}\log \frac{1}{\varepsilon}\\
	&\geq\max\left\{1+\frac{16}{3}\sqrt{\frac{\overline{L}\Mx}{\mu}},\frac{3}{8}\frac{\left(\sum_{\iclient=1}^\Mx \sqrt{\frac{L_\iclient}{\Mx}}\right)^2}{\sqrt{\overline{L}\mu\Mx}}+\frac{\sum_{\iclient=1}^\Mx \sqrt{\frac{L_\iclient}{\Mx}}}{\sqrt{\frac{L_{\min}}{\Mx}}}\right\}\log \frac{1}{\varepsilon}\\
	&\geq \max\left\{1+\frac{16}{3}\sqrt{\frac{\overline{L}\Mx}{\mu}},\max_m\left(\frac{3}{8}\frac{\left(\sum_{\iclient=1}^\Mx \sqrt{\frac{L_\iclient}{\Mx}}\right)^2}{\frac{L_\iclient}{\Mx}\sqrt{L^+\mu\Mx}}\frac{L_\iclient-\mu_\iclient}{\Mx}+\frac{\sum_{\iclient=1}^\Mx \sqrt{\frac{L_\iclient}{\Mx}}}{\sqrt{\frac{L_\iclient}{\Mx}}}\right)\right\}\log \frac{1}{\varepsilon}\\
	&\geq \max\left\{1+\frac{1}{\gammaM\mu},\max_m\left(\frac{1}{\widehat{p}_\iclient}\left(\frac{L_{F_\iclient}}{\tauM_\iclient}+1\right)\right)\right\}\log \frac{1}{\varepsilon}\\
	&\geq \max\left\{1+\frac{1}{\gammaM\mu},\max_m\left(\left(1+\qm\right)\left(\frac{L_{F_\iclient}}{\tauM_\iclient}+1\right)\right)\right\}\log \frac{1}{\varepsilon}.
\end{align*}
Unrolling the recursion from Theorem~\ref{thm:5GCMult} we get 
\begin{equation}
	\label{eq:5.5}
	\Exp{\Psi^{T}}\leq \left( \max\left\{\frac{1}{1+\gammaM\mu} , \max_{m}\left[\frac{ L_{F_\iclient}+\left(1-\widehat{p}_\iclient\right)\tauM_\iclient}{L_{F_\iclient}+\tauM_\iclient} \right\}\right]\right)^T\Psi^{0}.
\end{equation}
Using Lemma from \citet{malinovsky2021random} for recursion (Appendix B), we can state that derived $T$ is sufficient to guarantee \ref{eq:5.5}.
\end{proof}

\section{Independent Sampling (Sampling without Replacement)}
\subsection{Proof of Lemma~\ref{lem:ind}}
\begin{lemma*}
	The Independent Sampling with estimator~\ref{eq:1} satisfies the Assumption~\ref{weighted} with $A=\frac{1}{\sum_{m}^{M}\frac{\prm}{1-\prm}}$, $B=0$   and $w_m=\frac{\frac{\prm}{1-\prm}}{\sum_{m=1}^{M}\frac{\prm}{1-\prm}}$.
\end{lemma*}

	\begin{proof} The proof is presented in \citet{PAGE-AS}. Let us provide it for completeness.

Let us define i.i.d. random variables
\begin{eqnarray*}
	\chi_\iclient=\begin{cases}
		1 & \text{with probability }p_m\\
		0 & \text{with probability }1-p_m,
	\end{cases}
\end{eqnarray*}
for all $m\in[M]$, also take $S^t\eqdef\left\{m\in[M]|\chi_\iclient=1\right\}$ and $\underline{p} = \left(p_1,\ldots, p_M\right)$ .  The corresponding estimator for this sampling has the following form:  
\begin{equation}
	S(a_1,\ldots,a_M,\psi, \underline{p})\eqdef \frac{1}{\Mx}\sum_{\iclient\in S}\frac{a_m}{p_{m}}.
\end{equation}
We get
$$
\begin{aligned}
	\mathrm{E}\left[\left\|\frac{1}{M} \sum_{m \in S} \frac{a_m}{p_m}-\frac{1}{M} \sum_{m=1}^M a_m\right\|^2\right] & =\mathrm{E}\left[\left\|\frac{1}{M} \sum_{m=1}^M \frac{1}{p_m} \chi_m a_m\right\|^2\right]-\left\|\frac{1}{M} \sum_{m=1}^M a_m\right\|^2 \\
	& =\sum_{m=1}^M \frac{\mathrm{E}\left[\chi_m\right]}{M^2 p_m^2}\left\|a_m\right\|^2+\sum_{m \neq k} \frac{\mathrm{E}\left[\chi_m\right] \mathrm{E}\left[\chi_k\right]}{M^2 p_m p_k}\left\langle a_m, a_k\right\rangle\\
	&-\left\|\frac{1}{M} \sum_{m=1}^M a_m\right\| ^2 \\
	& =\sum_{m=1}^M \frac{1}{M^2 p_m}\left\|a_m\right\|^2+\frac{1}{M^2}\left(\left\|\sum_{m=1}^M a_m\right\|^2-\sum_{m=1}^M\left\|a_m\right\|^2\right)\\
	&-\left\|\frac{1}{M} \sum_{m=1}^M a_m\right\|^2 \\
	& =\frac{1}{M^2} \sum_{m=1}^M\left(\frac{1}{p_m}-1\right)\left\|a_m\right\|^2 .
\end{aligned}
$$
Thus we have $A=\frac{1}{\sum_{m=1}^M \frac{p_m}{1-p_m}}, B=0$ and $w_m=\frac{\frac{p_n}{1-p_m}}{\sum_{m=1}^M \frac{p_m}{1-p_m}}$ for all $m \in[M]$.
	\end{proof}

\subsection{Proof of Theorem~\ref{thm:5GCSINDS}}
\begin{theorem*}
	Consider Algorithm~\ref{alg:5GCS-AB} with Independent Sampling with estimator~\ref{eq:1} satisfying Assumption~\ref{weighted} and LT solver satisfying Assumption~\ref{assm:localtr}. Let the inequality hold $\frac{1}{\tauM_\iclient}-\left(\gammaM \Mx +\gammaM\frac{1-\prm}{\prm}\right)\geq \frac{4}{\tauM_\iclient^2}\frac{\mu_\iclient}{3M}$, e.g. $\tauM_\iclient\geq\frac{8\mu_\iclient}{3\Mx}$ and $\gammaM\leq\frac{1}{2\tauM_\iclient\left(\Mx +\frac{1-\prm}{\prm}\right)}$.  Then for the Lyapunov function 
	\begin{equation*}
		  	\Psi^{\kstep}\eqdef \frac{1}{\gammaM}\sqnorm{x^{t+1}-x^\star}+\sum_{\iclient=1}^\Mx\frac{1}{\prm}\left(\frac{1}{\tauM_\iclient}+\frac{1}{L_{F_\iclient}}\right)\sqnorm{u_\iclient^{t+1}-u_\iclient^\star},
	\end{equation*}
	the iterates of  the method satisfy
	\begin{equation*}
	 	\Exp{\Psi^{t+1}}\leq  \max\left\{\frac{1}{1+\gammaM\mu} ,\max_{m}\left[\frac{ L_{F_\iclient}+\left(1-\prm\right)\tauM_\iclient}{L_{F_\iclient}+\tauM_\iclient}\right] \right\}\Exp{\Psi^{t}},
	\end{equation*}
	where $\prm$ is probability that $m$-th client is participating. 
\end{theorem*}

\begin{proof}
	Using Lemma~\ref{lem:ind} we have $A=\frac{1}{\sum_{m=1}^M \frac{p_m}{1-p_m}}, B=0$ and $w_m=\frac{\frac{p_n}{1-p_m}}{\sum_{m=1}^M \frac{p_m}{1-p_m}}$ for all $m \in[M]$. Using Theorem~\ref{thm:5GCS-AB} and we plug this constants into 
			\begin{equation*}
	 	\Exp{\Psi^{t+1}}\leq  \max\left\{\frac{1}{1+\gammaM\mu} ,\frac{ L_{F_\iclient}+\frac{\qm}{1+\qm}\tauM_\iclient}{L_{F_\iclient}+\tauM_\iclient} \right\}\Exp{\Psi^{t}},
	\end{equation*} 
and we obtain  
	\begin{equation*}
 	\Exp{\Psi^{t+1}}\leq  \max\left\{\frac{1}{1+\gammaM\mu} ,\max_{m}\left[\frac{ L_{F_\iclient}+\left(1-\prm\right)\tauM_\iclient}{L_{F_\iclient}+\tauM_\iclient}\right] \right\}\Exp{\Psi^{t}}.
\end{equation*}
\end{proof}
\subsection{Proof of Corollary \ref{cor:5GCSINDS}}
\begin{corollary}
	Choose any $0<\varepsilon<1$ and $p_\iclient$ can be estimated but not set, then set $\tauM_\iclient =\frac{8}{3}\sqrt{\frac{\bar{L}\mu}{M\sum_{\iclient=1}^\Mx \prm}}$ and $\gammaM=\frac{1}{2\tauM_\iclient\left(\Mx +\frac{1-\prm}{\prm}\right)}$.  In order to guarantee $\Exp{\Psi^{\Tx}}\leq \varepsilon \Psi^0$, it suffices to take 
	\begin{eqnarray*}
		T \geq	
		\max\left\{1+\frac{16}{3}\sqrt{\frac{\overline{L}\Mx}{\mu\sum_{\iclient=1}^\Mx \prm}}\left(1+\frac{1}{\Mx}\frac{1-\prm}{\prm}\right),\max_{m}\left[\frac{3}{8}\frac{L_{F_m}}{\prm}\sqrt{\frac{\Mx\sum_{\iclient=1}^\Mx \prm}{\overline{L}\mu}}+\frac{1}{\prm}\right]\right\}\log \frac{1}{\varepsilon}
	\end{eqnarray*}
	communication rounds.
\end{corollary}
\begin{proof}
	First note that $\tauM_\iclient =\frac{8}{3}\sqrt{\frac{\overline{L}\mu}{M\sum_{\iclient=1}^\Mx \prm}} \geq \frac{8\mu}{3\Mx}$ and $\gammaM = \frac{3}{16}\sqrt{\frac{M\sum_{\iclient=1}^\Mx \prm}{\overline{L}\mu}}\frac{1}{\left(\Mx +\frac{1-\prm}{\prm}\right)}\leq\frac{1}{2\tauM_\iclient\left(\Mx +\frac{1-\prm}{\prm}\right)}$, thus the stepsizes choices satisfy $\frac{1}{\tauM_\iclient}-\left(\gammaM \Mx +\gammaM\frac{1-\prm}{\prm}\right)\geq \frac{4}{\tauM_\iclient^2}\frac{\mu_\iclient}{3M}$.
	Now we get the contraction constant from Theorem \ref{cor:5GCSINDS} to be equal to:
	\begin{eqnarray} 
 	1-\rho = 	\max\left\{1-\frac{\gammaM\mu}{1+\gammaM\mu},\max_{m}\left[1-\frac{\prm\tauM_m}{L_{F_m}+\tauM_m}\right]\right\}.\notag 
	\end{eqnarray}
	Let us derive the complexity:
	\begin{align*}
	T&\geq 	\max\left\{1+\frac{16}{3}\sqrt{\frac{\overline{L}\Mx}{\mu\sum_{\iclient=1}^\Mx \prm}}\left(1+\frac{1-\prm}{\Mx\prm}\right),\max_{m}\left[\frac{3}{8}\frac{L_{F_m}}{\prm}\sqrt{\frac{\Mx\sum_{\iclient=1}^\Mx \prm}{\overline{L}\mu}}+\frac{1}{\prm}\right]\right\} \log\frac{1}{\varepsilon}\\
	&\geq \max\left\{1+\frac{1}{\gammaM\mu},\max_{m}\left[\frac{L_{F_m}+\tauM_m}{\prm\tauM_m}\right]\right\} \log\frac{1}{\varepsilon}.
	\end{align*}
\end{proof}
\textbf{Remark.}	Note a very important special case, where $L_m=L$ and so $\overline{L}=L$ and $L_{F_m}=\frac{1}{M}(L-\mu)\leq L/M$. Choose $\prm$, so that $\sum_{\iclient=1}^M \prm=C$ (expected cohort size), then the above simplifies to 
	\begin{align*}
		  T=\max\left\{\max_{m}\left[1+\frac{16}{3}\sqrt{\frac{L\Mx}{\mu C}}\left(1+\frac{1}{\Mx}\frac{1-\prm}{\prm}\right)\right],\max_{m}\left[\frac{3}{8}\frac{1}{\prm}\sqrt{\frac{L\Cx}{M\mu}}+\frac{1}{\prm}\right]\right\} \log\frac{1}{\varepsilon} .
	\end{align*}
	Additionally specifying that $\prm=\frac{C}{M}$ gives
	\begin{align*}
		T & \geq\max\left\{1+\frac{16}{3}\sqrt{\frac{L\Mx}{\mu C}}\left(1+\frac{1}{\Mx}\frac{M-C}{C}\right),\frac{3}{8}\sqrt{\frac{L\Mx}{C\mu}}+\frac{M}{C}\right\} \log\frac{1}{\varepsilon}\\  & =	\mathcal{O}\left(\left(\frac{\Mx}{\Cx}+\sqrt{\frac{M}{C}\frac{L}{\mu }}\right)\log \frac{1}{\varepsilon} \right) .
	\end{align*}
\subsection{Tau-Nice sampling}
In this section we show that previous result of ~\citet{grudzien2023can} can be covered by our framework. This means we fully generalize previous convergence guarantees. 
\begin{theorem}\label{thm:5GCSTN}
	Consider Algorithm~\ref{alg:5GCS-AB} with uniform sampling scheme satisfying~\ref{weighted} and LT solver satisfying Assumption~\ref{assm:localtr}. Let the inequality hold $\frac{1}{\tauM_\iclient}-\gammaM \Mx \geq \frac{4}{\tauM_\iclient^2}\frac{\mu}{3M}$, e.g. $\tauM_\iclient\geq\frac{8\mu}{3\Mx}$ and $\gammaM\leq\frac{1}{2\tauM_\iclient\Mx }$.  Then for the Lyapunov function 
	\begin{equation*}
		  	\Psi^{\kstep}\eqdef \frac{1}{\gammaM}\sqnorm{x^{t+1}-x^\star}+\sum_{\iclient=1}^\Mx\frac{\Mx}{\Cx}\left(\frac{1}{\tauM_\iclient}+\frac{1}{L_{F_\iclient}}\right)\sqnorm{u_\iclient^{t+1}-u_\iclient^\star},
	\end{equation*}
	the iterates of  the method satisfy
	\begin{equation*}
	 	\Exp{\Psi^{t+1}}\leq  \max\left\{\frac{1}{1+\gammaM\mu} ,\max_{m}\left[\frac{ L_{F_\iclient}+\frac{\Mx-\Cx}{\Mx}\tauM_\iclient}{L_{F_\iclient}+\tauM_\iclient}\right] \right\}\Exp{\Psi^{t}}.
	\end{equation*}
\end{theorem}

\begin{corollary}\label{cor:5GCSTN}
	Suppose that $L_m=L,\forall \iclient\in\left\{1,\ldots,\Mx\right\}$. Choose any $0<\varepsilon<1$ and $\gammaM = \frac{3}{16}\sqrt{\frac{\Cx}{L\mu \Mx}}$ and $\tauM_m=\frac{8}{3}\sqrt{\frac{L\mu}{\Mx\Cx}} $. In order to guarantee $\Exp{\Psi^{\Tx}}\leq \varepsilon \Psi^0$, it suffices to take 
	\begin{eqnarray}
		 	T &\geq	&  
			\max\left\{1+\frac{16}{3}\sqrt{\frac{\Mx}{\Cx}\frac{L}{\mu}},\frac{\Mx}{\Cx}+\frac{3}{8}\sqrt{\frac{\Mx}{\Cx}\frac{L}{\mu} }\right\}\log\frac{1}{\varepsilon} \notag \\
		&=&   
		\cO\left(\left(\frac{\Mx}{\Cx}+\sqrt{\frac{\Mx}{\Cx}\frac{L}{\mu }}\right) \log \frac{1}{\varepsilon} \right)\notag
	\end{eqnarray}
	communication rounds.
\end{corollary}
\subsection{Proof of Corollary \ref{cor:5GCSTN}}
\begin{proof}
	First note that $\tauM_m=\tauM=\frac{8}{3}\sqrt{\frac{L\mu}{\Mx\Cx}} \geq \frac{8\mu}{3\Mx}$ and $\gammaM = \frac{3}{16}\sqrt{\frac{\Cx}{L\mu \Mx}}\geq\frac{1}{2\tauM_\iclient\Mx }$, thus the stepsizes choices satisfy $\frac{1}{\tauM_\iclient}-\gammaM \Mx \geq \frac{4}{\tauM_\iclient^2}\frac{\mu}{3M}$.
	Now we get the contraction constant from Theorem \ref{thm:5GCSTN} to be equal to:
	\begin{eqnarray}
		1-\rho =	\mathcal{O}\left(\max\left\{1-\frac{\gammaM\mu}{1+\gammaM\mu},1-\frac{\frac{\Cx}{\Mx}\tauM}{L_{F_m}+\tauM}\right\}\right).\notag 
	\end{eqnarray}
	This  gives a rate of 
	\begin{align*}
		T &= \max\left\{1+\frac{1}{\gammaM\mu},\frac{\Mx}{\Cx}\frac{L/M+\tauM}{\tauM}\right\} \log\frac{1}{\varepsilon}\\
		&=
		\max\left\{1+\frac{16}{3}\sqrt{\frac{L\Mx}{\mu\Cx}},\frac{\Mx}{\Cx}+ \frac{3}{8}\sqrt{\frac{L\Mx}{\mu\Cx}}\right\} \log\frac{1}{\varepsilon}\\
		&=	\mathcal{O}\left(\left(\frac{\Mx}{\Cx}+\sqrt{\frac{L\Mx}{\mu \Cx}}\right)\log \frac{1}{\varepsilon} \right) .
	\end{align*}
\end{proof}

\clearpage
\section{Analysis of 5GCS-CC}
\subsection{Proof of Theorem~\ref{thm:inexactCompr}}
In this section we will provide the proof for general version of \algname{5GCS} algorithm, which is Algorithm~\ref{alg:5GCSnew}. This method is inexact version of \algname{RandProx} presented in \citet{RandProx}.

We need to formulate an assumption similar to Assumption~\ref{weighted}.  
\begin{algorithm*}[!t]
	\caption{\algn{inexact-RandProx}}
	\footnotesize
	\begin{algorithmic}[1]\label{alg:5GCSnew}
		\STATE  \textbf{Input:} initial primal iterates $x^0\in\mathbb{R}^d$; initial dual iterates $u_1^0, \dots,u_{\Mx}^0 \in\mathbb{R}^d$; primal stepsize $\gammaM>0$; dual stepsize $\tauM>0$
		\STATE  \textbf{Initialization:}  $v^0\eqdef \sum_{\iclient=1}^\Mx u_\iclient^0$  \hfill {\color{gray} \footnotesize $\diamond$ The server initiates $v^{0}$ as the sum of the initial dual iterates} 
		\FOR{communication round $t=0, 1, \ldots$} 
		\STATE Compute $\hat{x}^{t} = \frac{1}{1+\gammaM\mu} \left(x^t - \gammaM v^t\right)$ and broadcast it to the clients 
		\STATE Find $\lastlocitterk$ as the final point after $\Kx$ iterations of some local optimization algorithm $\mathcal{A}$, initiated with $y^0=\Koper\hat{x}^t$, for solving the optimization problem 
		\begin{eqnarray}
			 			\lastlocitterk \approx \argmin \limits_{y\in\mathbb{R}^{d\Mx}}\left\{\psi^t(y) \eqdef  F( y)+\frac{\tauM}{2} \sqnorm{ y-\left(\Koper\hat{x}^\kstep+\frac{1}{\tauM}u^t\right)}\right\}\label{localprob}
		\end{eqnarray}
		\STATE Compute $\bar{u}^{t+1}=\nabla F(\lastlocitterk)$ and send $\Ropp^t\left(\bar{u}^{t+1}-u^t\right)$ to the server
		\STATE  $u^{t+1}= u^t + \tfrac{1}{1+\omega} \Ropp^t\left(\bar{u}^{t+1}-u^t\right)$
		\STATE $v^{t+1} \eqdef \sum_{\iclient=1}^{\Mx} u_\iclient^{t+1} $ \hfill {\color{gray} \footnotesize $\diamond$ The server maintains $v^{t+1}$ as the sum of the dual iterates} 
		\STATE  $x^{t+1} \eqdef \hat{x}^{t}- \gammaM \left(1+\omega\right) (v^{t+1}-v^t)$ \hfill {\color{gray} \footnotesize $\diamond$ The server updates the primal iterate} 
		\ENDFOR
	\end{algorithmic}
\end{algorithm*}

\begin{assumption}\label{AB}
	
	(AB Inequality). Let $\Ropp : \mathbb{R}^{dM} \rightarrow \mathbb{R}^{dM} $, be an unbiased random operator which satisfies:
	\begin{equation}
	 	\Exp{\sqnorm{\Koper^\top\left(\Ropp(v)-v\right)}}\leq A\sum_{\iclient=1}^{\Mx}\sqnorm{v_\iclient}-B\sqnorm{\sum_{\iclient=1}^{\Mx}v_\iclient},\label{PP}
	\end{equation} 
	for some $A,B>0$,
	where $v=\left(v_1,\dots,v_\Mx\right)^\top$ and $v_{\iclient}\in\mathbb{R}^d$ for $\iclient\in\left\{1,\dots,\Mx\right\}$.

\end{assumption}

\begin{theorem} \label{thm:inexact}
	Consider Algorithm~\ref{alg:5GCSnew} (\algname{Inexact-RandProx}) with the LT solver satisfying Assumption~\ref{assm:localtr}.
	Let $\frac{1}{\tauM}-(\gammaM(1-B)\Mx+\gammaM A))\geq \frac{4}{\tauM^2}\frac{\mu}{3\Mx}$, e.g. $\tauM\geq\frac{8\mu}{3\Mx}$ and $\gammaM=\frac{1}{2\tauM\left(\Mx+A-\Mx B\right)}$. Then for the Lyapunov function
	\begin{equation*}
		  	\Psi^{\kstep}\eqdef \frac{1}{\gammaM}\sqnorm{x^{\kstep}-x^\star}+(1+\omega)\left(\frac{1}{\tauM}+\frac{1}{L_F}\right)\sqnorm{u^{\kstep}-u^\star},
	\end{equation*}
	the iterates of  the method satisfy
	$		\Exp{\Psi^{\Tx}}\leq (1-\rho)^\Tx \Psi^0,
	$	where 
	$
	\rho\eqdef  \min\left\{\frac{\gammaM\mu}{1+\gammaM\mu},\frac{1}{1+\omega}\frac{\tauM}{(L_F+\tauM)}\right\}<1.
	$
	
\end{theorem}
\begin{proof}
	Noting that updates for $u^{t+1}$ and $x^{t+1}$ can be written as
	\begin{eqnarray}
		&u^{t+1}\eqdef u^t + \tfrac{1}{1+\omega} \Ropp^t\left(\bar{u}^{t+1}-u^t\right),&\\
		&x^{t+1}=\hat{x}^t-\gammaM\left(\omega+1\right)\Koper^\top\left(u^{t+1}-u^t\right),&\label{xupdate35}
	\end{eqnarray}
	where $\Ropp^t$ is any random operator, which satisfies conic variance (in this case it is not compression parameter) and Assumption~\ref{AB} and  $\bar{u}^{t+1} = \nabla F(\lastlocitterk)$. 
	Then using variance decomposition and proposition 1 from \cite{condat2021murana} we obtain
	\begin{eqnarray}
		\label{eq:starting_eq3.5}
		\Exp{\sqnorm{x^{t+1}-x^\star}\;|\;\mathcal{F}_t}&\overset{(\ref{vardec})}{=}&\sqnorm{\Exp{x^{t+1}\;|\;\mathcal{F}_t}-x^\star}+\Exp{\sqnorm{x^{t+1}-\Exp{x^{t+1}\;|\;\mathcal{F}_t}}\;|\;\mathcal{F}_t}\notag\\
		&\overset{(\ref{xupdate35})+(\ref{weighted})}{=}& \underbrace{\sqnorm{\hat{x}^{t}-x^\star-\gammaM \Koper^\top(\bar{u}^{t+1}-u^t)}}_{X}+\gammaM^2A\sqnorm{\bar{u}^{t+1}-u^t}\notag\\
		&\quad &- \gammaM^2B\sqnorm{\Koper^\top(\bar{u}^{t+1}-u^t)}.
	\end{eqnarray}
	Moreover, using \myref{fooc} and the definition of $\hat{x}^t$, we have
	\begin{align}
		&(1+\gammaM\mu)\hat{x}^t=x^t-\gammaM \Koper^\top u^{t},\label{opt1}\\
		&(1+\gammaM\mu)x^\star= x^\star -\gammaM \Koper^\top u^\star.\label{opt2}
	\end{align}
	Using \myref{opt1} and \myref{opt2} we obtain
	\begin{eqnarray}
		\label{eq:long3.5}
		X &=&	\sqnorm{\hat{x}^{t}-x^\star-\gammaM \Koper^\top(\bar{u}^{t+1}-u^t)}\notag\\
		&=&\sqnorm{\hat{x}^{t}-x^\star} +\gammaM^2\sqnorm{\Koper^\top (\bar{u}^{t+1}-u^t)}\notag\\
		&\quad&-2\gammaM \left\langle 
		\hat{x}^{t}-x^\star,\Koper^\top(\bar{u}^{t+1}-u^t) \right\rangle  \notag\\
		&=& (1+\gammaM\mu) \sqnorm{\hat{x}^{t}-x^\star} +\gammaM^2\sqnorm{\Koper^\top(\bar{u}^{t+1}-u^t)}\notag\\
		&\quad&-2\gammaM  \left\langle 
		\hat{x}^{t}-x^\star,\Koper^\top (\bar{u}^{t+1}-u^\star) \right\rangle  +2\gammaM  \left\langle 
		\hat{x}^{t}-x^\star,\Koper^\top (u^{t}-u^\star) \right\rangle  \notag\\
		&\quad&-\gammaM\mu\sqnorm{\hat{x}^{t}-x^\star} \notag\\
		&\overset{(\ref{opt1})+(\ref{opt2})}{=}&   \left\langle  x^t-x^\star-\gammaM \Koper^\top (u^t-u^\star),\hat{x}^{t}-x^\star \right\rangle +\gammaM^2\sqnorm{\Koper^\top (\bar{u}^{t+1}-u^t)}\notag\\
		&\quad&-2\gammaM  \left\langle 
		\hat{x}^{t}-x^\star,\Koper^\top (\bar{u}^{t+1}-u^\star) \right\rangle  +  \left\langle 
		\hat{x}^{t}-x^\star,2\gammaM \Koper^\top (u^{t}-u^\star) \right\rangle  \notag\\
		&\quad&-\gammaM\mu\sqnorm{\hat{x}^{t}-x^\star} \notag.
	\end{eqnarray}
	It leads to	
	\begin{eqnarray}
		X	&= &\left\langle  x^t-x^\star+\gammaM \Koper^\top (u^t-u^\star),\hat{x}^{t}-x^\star \right\rangle  \notag\\
		&\quad&+\gammaM^2\sqnorm{\Koper^\top (\bar{u}^{t+1}-u^t)}-2\gammaM  \left\langle 
		\hat{x}^{t}-x^\star, \Koper^\top(\bar{u}^{t+1}-u^\star) \right\rangle  \notag\\
		&\quad&-\gammaM\mu\sqnorm{\hat{x}^{t}-x^\star}\notag \\
		&\overset{(\ref{opt1})+(\ref{opt2})}{=}&\frac{1}{1+\gammaM\mu} \left\langle  x^t-x^\star+\gammaM \Koper^\top (u^t-u^\star),x^t-x^\star-\gammaM \Koper^\top (u^t-u^\star) \right\rangle  \notag\\
		&\quad&+\gammaM^2\sqnorm{\Koper^\top (\bar{u}^{t+1}-u^t)}-2\gammaM  \left\langle 
		\hat{x}^{t}-x^\star, \Koper^\top (\bar{u}^{t+1}-u^\star) \right\rangle  \notag\\
		&\quad&-\gammaM\mu\sqnorm{\hat{x}^{t}-x^\star} \notag\\
		&=& \frac{1}{1+\gammaM\mu}\sqnorm{x^t-x^\star}-\frac{\gammaM^2}{1+\gammaM\mu}\sqnorm{\Koper^\top (u^t-u^\star)}\notag\\
		&\quad&+\gammaM^2\sqnorm{\Koper^\top (\bar{u}^{t+1}-u^t)}-2\gammaM  \left\langle 
		\hat{x}^{t}-x^\star,\Koper^\top (\bar{u}^{t+1}-u^\star) \right\rangle\notag\\
		&\quad&-\gammaM\mu\sqnorm{\hat{x}^{t}-x^\star}   . 
	\end{eqnarray}
	Combining \myref{eq:starting_eq3.5} and \myref{eq:long3.5} we have
	\begin{eqnarray*}
		\Exp{\sqnorm{x^{t+1}-x^\star}\;|\;\mathcal{F}_t}&\leq&  \frac{1}{1+\gammaM\mu}\sqnorm{x^t-x^\star}-\frac{\gammaM^2}{1+\gammaM\mu}\sqnorm{\Koper^\top (u^t-u^\star)}\\
		&\quad&+\gammaM^2(1-B)\sqnorm{\Koper^\top (\bar{u}^{t+1}-u^t)}-2\gammaM  \left\langle 
		\hat{x}^{t}-x^\star,\Koper^\top (\bar{u}^{t+1}-u^\star) \right\rangle \\
		&\quad&+\gammaM^2A\sqnorm{\bar{u}^{t+1}-u^t}-\frac{\gammaM\mu}{\Mx}\sqnorm{\Koper\hat{x}^t-\Koper x^\star}.
	\end{eqnarray*}
	Note that we can have the update rule for $u$ as: 
	\begin{equation*}
		u^{t+1}\eqdef u^t + \tfrac{1}{1+\omega} \Ropp^t\left(\bar{u}^{t+1}-u^t\right).
	\end{equation*}
	Using conic variance formula \myref{cvar} of $\Ropp^t$ we obtain
	\begin{align}
		\Exp{\sqnorm{u^{t+1}-u^\star}\;|\;\mathcal{F}_t}&\overset{(\ref{vardec})+(\ref{cvar})}{\leq} \sqnorm{u^{t}-u^\star+\frac{1}{1+\omega}\left(\bar{u}^{t+1} -u^t\right)}
		+\frac{\omega}{(1+\omega)^2}\sqnorm{\bar{u}^{t+1} -u^t }\notag\\
		&=\frac{\omega^2}{(1+\omega)^2}\sqnorm{u^{t}-u^\star}+\frac{1}{(1+\omega)^2}\sqnorm{\bar{u}^{t+1}-u^\star}\notag\\
		&+\frac{2\omega}{(1+\omega)^2} \left\langle  u^{t}-u^\star,
		\bar{u}^{t+1}-u^\star \right\rangle +\frac{\omega}{(1+\omega)^2}\sqnorm{\bar{u}^{t+1} -u^\star }\notag\\
		&+\frac{\omega}{(1+\omega)^2}\sqnorm{u^{t} -u^\star }-\frac{2\omega}{(1+\omega)^2} \left\langle  u^{t}-u^\star,
		\bar{u}^{t+1}-u^\star \right\rangle \notag\\
		&=\frac{1}{1+\omega}\sqnorm{\bar{u}^{t+1} -u^\star }+\frac{\omega}{1+\omega}\sqnorm{u^{t} -u^\star }.\label{expu}
	\end{align}
	Let us consider the first term in \myref{expu}:
	\begin{eqnarray*}
		\sqnorm{\bar{u}^{t+1}-u^\star}&=&\sqnorm{(u^t-u^\star)+(\bar{u}^{t+1}-u^t)}\\
		&=&\sqnorm{u^t-u^\star}+\sqnorm{\bar{u}^{t+1}-u^t}+2 \left\langle  u^t-u^\star,\bar{u}^{t+1}-u^t \right\rangle \\
		&=&\sqnorm{u^t-u^\star}+2  \left\langle  \bar{u}^{t+1}-u^\star,\bar{u}^{t+1}-u^t \right\rangle  - \sqnorm{\bar{u}^{t+1}-u^t}.
	\end{eqnarray*}
	Combining terms together we get
	\begin{equation*}
		\Exp{\sqnorm{u^{t+1}-u^\star}\;|\;\mathcal{F}_t}\leq \sqnorm{u^{t} -u^\star }+\frac{1}{1+\omega}\left(2  \left\langle  \bar{u}^{t+1}-u^\star,\bar{u}^{t+1}-u^t \right\rangle  - \sqnorm{\bar{u}^{t+1}-u^t}\right).
	\end{equation*}
	Finally, we obtain
	\begin{eqnarray*}
		\frac{1}{\gammaM}\Exp{\sqnorm{x^{t+1}-x^\star}\;|\;\mathcal{F}_t}&+&\frac{1+\omega}{\tauM}\Exp{\sqnorm{u^{t+1}-u^\star}\;|\;\mathcal{F}_t}\\
		&\leq&  \frac{1}{\gammaM(1+\gammaM\mu)}\sqnorm{x^t-x^\star}-\frac{\gammaM}{1+\gammaM\mu}\sqnorm{\Koper^\top (u^t-u^\star)}\\
		&\quad&+\gammaM(1-B)\sqnorm{\Koper^\top (\bar{u}^{t+1}-u^t)}\\
		&\quad&+\gammaM A\sqnorm{\bar{u}^{t+1}-u^t}-\frac{\mu}{\Mx}\sqnorm{\Koper \hat{x}^t- \Koper x^\star}\\
		&\quad&+\frac{1+\omega}{\tauM}\sqnorm{u^{t} -u^\star }-2  \left\langle 
		\hat{x}^{t}-x^\star,\Koper^\top (\bar{u}^{t+1}-u^\star) \right\rangle \\ &\quad&+\frac{1}{\tauM}\left(2  \left\langle  \bar{u}^{t+1}-u^\star,\bar{u}^{t+1}-u^t \right\rangle  - \sqnorm{\bar{u}^{t+1}-u^t}\right).
	\end{eqnarray*}
	Ignoring $-\frac{\gammaM}{1+\gammaM\mu}\sqnorm{\Koper^\top (u^t-u^\star)}$ and noting
	\begin{align*}
		&-\left\langle \hat{x}^{t}-x^\star,\Koper^\top (\bar{u}^{t+1}-u^\star) \right\rangle +\frac{1}{\tauM}\left\langle  \bar{u}^{t+1}-u^\star,\bar{u}^{t+1}-u^t \right\rangle  \\
		&=-\left\langle \lastlocitterk-\Koper x^\star,\bar{u}^{t+1}-u^\star \right\rangle +\frac{1}{\tauM}\left\langle  \nabla\localfun (\lastlocitterk),\bar{u}^{t+1}-u^\star \right\rangle\\
		&\overset{(\ref{yi1})+(\ref{strmono})}{\leq} -\frac{1}{L_F}\sqnorm{\bar{u}^{t+1}-u^\star}+\frac{a}{2\tauM}\sqnorm{\nabla\localfun (\lastlocitterk)}+\frac{1}{2a\tauM}\sqnorm{\bar{u}^{t+1}-u^\star }\\
		&= -\left(\frac{1}{L_F}-\frac{1}{2a\tauM}\right)\sqnorm{\bar{u}^{t+1}-u^\star}+\frac{a}{2\tauM}\sqnorm{\nabla\localfun (\lastlocitterk)}\\
		&\overset{(\ref{expu})}{\leq} -\left(\frac{1}{L_F}-\frac{1}{2a\tauM}\right)\left((1+\omega)\Exp{\sqnorm{u^{t+1}-u^\star}\;|\;\mathcal{F}_t}-\omega\sqnorm{u^t-u^\star}\right)\\
		&+\frac{a}{2\tauM}\sqnorm{\nabla\localfun (\lastlocitterk)},
	\end{align*}
	we get
	\begin{eqnarray*}
		\frac{1}{\gammaM}\Exp{\sqnorm{x^{t+1}-x^\star}\;|\;\mathcal{F}_t}&+&\left(1+\omega\right)\left(\frac{1}{\tauM}+\frac{1}{L_F}\right)\Exp{\sqnorm{u^{t+1}-u^\star}\;|\;\mathcal{F}_t}\\
		&\leq&  \frac{1}{\gammaM(1+\gammaM\mu)}\sqnorm{x^t-x^\star}\\
		&\quad&+\left(1+\omega\right)\left(\frac{1}{\tauM}+\frac{\omega}{1+\omega}\frac{1}{L_F}\right)\sqnorm{u^{t} -u^\star } \\
		&\quad&+\left(\gammaM\left(1-B\right)\Mx+\gammaM A-\frac{1}{\tauM}\right)\sqnorm{\bar{u}^{t+1}-u^t}\\
		&\quad& +\frac{L_F}{\tauM^2}\sqnorm{\nabla\localfun (\lastlocitterk)}-\frac{\mu}{\Mx}\sqnorm{\Koper\hat{x}^t-\Koper x^\star}.
	\end{eqnarray*}
	Where we made the choice $a=\frac{L_F}{\tau}$.
	Using Young's inequality we have
	\begin{equation*}
		-\frac{\mu}{3\Mx}\sqnorm{\Koper \hat{x}^t-\localsolk+\localsolk-\Koper x^\star}\overset{(\ref{yi3})}{\leq}\frac{\mu}{3\Mx}\sqnorm{\localsolk-\Koper x^\star}-\frac{\mu}{6\Mx}\sqnorm{\Koper\hat{x}^t-\localsolk}.	\end{equation*}
	Noting the fact that $\localsolk=\Koper\hat{x}^t-\frac{1}{\tauM}(\hat{u}^{t+1}-u^t)$, we have
	$$\frac{\mu}{3\Mx}\sqnorm{\localsolk-\Koper x^\star}\overset{(\ref{yi2})}{\leq} 2\frac{\mu}{3\Mx}\sqnorm{\Koper\hat{x}^t-\Koper x^\star}+\frac{2}{\tauM^2}\frac{\mu}{3\Mx}\sqnorm{\hat{u}^{t+1}-u^t}.$$
	Combining those inequalities we get
	\begin{eqnarray*}
		\frac{1}{\gammaM}\Exp{\sqnorm{x^{t+1}-x^\star}\;|\;\mathcal{F}_t}&+&\left(1+\omega\right)\left(\frac{1}{\tauM}+\frac{1}{L_F}\right)\Exp{\sqnorm{u^{t+1}-u^\star}\;|\;\mathcal{F}_t}\\
		&\leq&  \frac{1}{\gammaM(1+\gammaM\mu)}\sqnorm{x^t-x^\star}\\
		&\quad&+\left(1+\omega\right)\left(\frac{1}{\tauM}+\frac{\omega}{1+\omega}\frac{1}{L_F}\right)\sqnorm{u^{t} -u^\star } \\
		&\quad&+\frac{2}{\tauM^2}\frac{\mu}{3\Mx}\sqnorm{\hat{u}^{t+1}-u^t}\\
		&\quad&-\left(\frac{1}{\tauM}-\left(\gammaM\left(1-B\right)\Mx+\gammaM A\right)\right)\sqnorm{\bar{u}^{t+1}-u^t}\\
		&\quad&+\frac{L_F}{\tauM^2}\sqnorm{\nabla\localfun (\lastlocitterk)}-\frac{\mu}{6\Mx}\sqnorm{\Koper\hat{x}^t-\localsolk}.
	\end{eqnarray*}
	Assuming $\gammaM$ and $\tauM$ can be chosen so that $\frac{1}{\tauM}-(\gammaM(1-B)\Mx+\gammaM A))\geq \frac{4}{\tauM^2}\frac{\mu}{3\Mx}$ we obtain
	\begin{eqnarray*}
		\frac{1}{\gammaM}\Exp{\sqnorm{x^{t+1}-x^\star}\;|\;\mathcal{F}_t}&+&\left(1+\omega\right)\left(\frac{1}{\tauM}+\frac{1}{L_F}\right)\Exp{\sqnorm{u^{t+1}-u^\star}\;|\;\mathcal{F}_t}\\
		&\leq&  \frac{1}{\gammaM(1+\gammaM\mu)}\sqnorm{x^t-x^\star}\\
		&\quad&+\left(1+\omega\right)\left(\frac{1}{\tauM}+\frac{\omega}{1+\omega}\frac{1}{L_F}\right)\sqnorm{u^{t} -u^\star } \\
		&\quad&+\frac{4}{\tauM^2}\frac{\mu L_F^2}{3\Mx}\sqnorm{\lastlocitterk-\localsolk}+\frac{L_F}{\tauM^2}\sqnorm{\nabla\localfun (\lastlocitterk)}\\
		&\quad&-\frac{\mu}{6\Mx}\sqnorm{\Koper\hat{x}^t-\localsolk}.
	\end{eqnarray*}
	 The point $\lastlocitterk$ is assumed to satisfy Assumption\ref{assm:localtr}:
	\begin{eqnarray*}
		\frac{4}{\tauM^2}\frac{\mu L_F^2}{3\Mx}\sqnorm{\lastlocitterk-\localsolk}+\frac{L_F}{\tauM^2}\sqnorm{\nabla\localfun (\lastlocitterk)}\leq\frac{\mu}{6\Mx}\sqnorm{\Koper\hat{x}^t-\localsolk}.
	\end{eqnarray*}
	Thus
	\begin{eqnarray*}
		\frac{1}{\gammaM}\Exp{\sqnorm{x^{t+1}-x^\star}\;|\;\mathcal{F}_t}&+&\left(1+\omega\right)\left(\frac{1}{\tauM}+\frac{1}{L_F}\right)\Exp{\sqnorm{u^{t+1}-u^\star}\;|\;\mathcal{F}_t}\\
		&\leq&  \frac{1}{\gammaM(1+\gammaM\mu)}\sqnorm{x^t-x^\star}\\
		&\quad&+\left(1+\omega\right)\left(\frac{1}{\tauM}+\frac{\omega}{1+\omega}\frac{1}{L_F}\right)\sqnorm{u^{t} -u^\star }.
	\end{eqnarray*}
	
	By taking the expectation on both sides we get
	\begin{equation*}
		\Exp{\Psi^{t+1}}\leq  \max\left\{\frac{1}{1+\gammaM\mu} ,\frac{ L_F+\frac{\omega}{1+\omega}\tauM}{L_F+\tauM} \right\}\Exp{\Psi^{t}},
	\end{equation*}
	which finishes the proof. The requirement for stepsizes becomes:$$\frac{1}{\tauM}-\gammaM(\Mx+A-\Mx B)\geq \frac{4}{\tauM^2}\frac{\mu}{3\Mx}.$$
	This inequality can be satisfied. Firstly note that for any $\Ropp$ we need to have $A\geq\Mx B$.
	Then as long as $\tauM\geq\frac{8\mu}{3\Mx}$ we can set $\gammaM$ to satisfy $\gammaM=\frac{1}{2 \tauM \left(\Mx+A-\Mx B\right)}$.
\end{proof}

Given this inequality we can formulate a following convergence theorem for Algorithm \ref{alg:5GC-CC}, which is practically just a corollary to the Theorem \ref{thm:inexact}.

\begin{theorem*} 
	Consider Algorithm~\ref{alg:5GC-CC} (\algname{5GCS-CC}) with the LT solver satisfying Assumption~\ref{assm:localtr}.
	Let $\frac{1}{\tauM}-\gammaM(\Mx+\conic\frac{\Mx}{\Cx})\geq \frac{4}{\tauM^2}\frac{\mu}{3\Mx}$, e.g. $\tauM\geq\frac{8\mu}{3\Mx}$ and $\gammaM=\frac{1}{2\tauM\left(\Mx+\conic\frac{\Mx}{\Cx}\right)}$. Then for the Lyapunov function
	\begin{equation*}
		  	\Psi^{\kstep}\eqdef \frac{1}{\gammaM}\sqnorm{x^{\kstep}-x^\star}+\frac{\Mx}{\Cx}\left(\conic+1\right)\left(\frac{1}{\tauM}+\frac{1}{L_F}\right)\sqnorm{u^{\kstep}-u^\star},
	\end{equation*}
	the iterates satisfy
	$		\Exp{\Psi^{\Tx}}\leq (1-\rho)^\Tx \Psi^0,
	$	with 
	$
	\rho\eqdef  \min\left\{\frac{\gammaM\mu}{1+\gammaM\mu},\frac{C}{M(1+\omega)}\frac{\tauM}{(L_F+\tauM)}\right\}<1.$
	
\end{theorem*}
\begin{corollary*}
	Choose any $0<\varepsilon<1$ and $\tauM = \frac{8}{3}\sqrt{L\mu\left(\frac{\conic+1}{\Cx}\right)\frac{1}{\left(\Mx+\frac{\Mx}{\Cx}\conic\right)}}$ and $\gammaM=\frac{1}{2\tauM\left(\Mx+\conic\frac{\Mx}{\Cx}\right)}$. In order to guarantee $\Exp{\Psi^{\Tx}}\leq \varepsilon \Psi^0$, it suffices to take 
	\begin{eqnarray}		  
		T \geq	
		\cO\left(\left(\frac{\Mx}{\Cx}\left(\conic+1\right)+\left(\sqrt{\frac{\conic}{\Cx}}+1\right)\sqrt{\left(\conic+1\right)\frac{\Mx}{\Cx}\frac{L}{\mu}}\right) \log \frac{1}{\varepsilon} \right)\notag
	\end{eqnarray}
	communication rounds.
\end{corollary*}
\subsection{Proof of Corollary \ref{cor:5GCSC}}
\begin{proof}
	First note that $\tauM =\frac{8}{3}\sqrt{L\mu\left(\frac{\conic+1}{\Cx}\right)\frac{1}{\left(\Mx+\frac{\Mx}{\Cx}\conic\right)}} \geq \frac{8\mu}{3\Mx}$ and $\gammaM = \frac{3}{16}\sqrt{\frac{1}{L\mu}\left(\frac{\Cx}{\conic+1}\right)\frac{1}{\left(\Mx+\frac{\Mx}{\Cx}\conic\right)}}\geq\frac{1}{2\tauM\left(\Mx+\conic\frac{\Mx}{\Cx}\right)}$, thus the stepsizes choices satisfy $\frac{1}{\tauM}-\gammaM(\Mx+\conic\frac{\Mx}{\Cx})\geq \frac{4}{\tauM^2}\frac{\mu}{3\Mx}$.
	Now we get the contraction constant from Theorem \ref{thm:inexactCompr} to be equal to:
	\begin{eqnarray} 
	1-\rho = 	\max\left\{1-\frac{\gammaM\mu}{1+\gammaM\mu},1-\frac{C}{M}\frac{1}{\conic+1}\frac{\tauM}{L_{F}+\tauM}\right\}.\notag 
	\end{eqnarray}
	This  gives us a communication complexity, let us define $\lambda = \frac{M}{C}(\omega+1):$
	\begin{align*}
		T &=  		\cO\left(\left(\frac{\Mx}{\Cx}\left(\conic+1\right)+\left(\sqrt{\frac{\conic}{\Cx}}+1\right)\sqrt{\left(\conic+1\right)\frac{\Mx}{\Cx}\frac{L}{\mu}}\right)\log \frac{1}{\varepsilon}\right) \\
		&\geq   \max\left\{1+\frac{16}{3}\left(\sqrt{\frac{\conic}{\Cx}}+1\right)\sqrt{\left(\conic+1\right)\frac{\Mx}{\Cx}\frac{L}{\mu}},\lambda+\frac{3}{8}\left(\sqrt{\frac{\conic}{\Cx}}+1\right)\sqrt{\left(\conic+1\right)\frac{\Mx}{\Cx}\frac{L}{\mu}}\right\}\log\frac{1}{\varepsilon} \notag \\ &= 
		\max\left\{1+\frac{16}{3}\sqrt{\frac{L}{\mu}\frac{\conic+1}{C}\left(\Mx+\frac{\Mx\conic}{\Cx}\right)},\lambda\left(1+\frac{L}{M}\frac{3}{8}\sqrt{\frac{1}{L\mu}\frac{M(C+\conic)}{\conic+1}}\right)\right\} \log\frac{1}{\varepsilon} \\
		&\geq\max\left\{1+\frac{1}{\gammaM\mu},\left(\conic+1\right)\frac{M}{C}\frac{L/M+\tauM}{\tauM}\right\} \log\frac{1}{\varepsilon}.
	\end{align*}
\end{proof}
\clearpage

\end{document}